\def\year{2017}\relax
\documentclass[letterpaper]{article}
\usepackage{aaai17}
\usepackage{times}
\usepackage{helvet}
\usepackage{courier}
\usepackage{url}
\usepackage{mdwmath}
\usepackage{mathrsfs}
\usepackage{multirow}
\usepackage{mdwtab}
\usepackage{amssymb}
\usepackage{amsthm}
\usepackage{amsmath}
\usepackage{array}
\usepackage{amsfonts}
\usepackage{algorithm}
\usepackage{algorithmic}
\usepackage[tight,footnotesize]{subfigure}
\usepackage{eqparbox}
\usepackage{enumerate}
\usepackage{framed}
\usepackage{fancyhdr}
\usepackage{graphicx}
\usepackage{lipsum}
\usepackage{enumitem}

\newcommand{\citet}[1]{\citeauthor{#1}~\shortcite{#1}}
\newcommand{\citep}{\cite}

\newtheorem{theorem}{Theorem}
\newtheorem{property}{Property}
\newtheorem{lemma}{Lemma}

\newtheorem{remark}{Remark}
\newtheorem{proposition}{Proposition}
\begin{document}
%
\title{Accelerated Variance Reduced Stochastic ADMM}

\author{Yuanyuan Liu, Fanhua Shang, James Cheng\\
Department of Computer Science and Engineering, The Chinese University of Hong Kong\\
$\{$yyliu, fhshang, jcheng$\}$@cse.cuhk.edu.hk
}
\maketitle

\begin{abstract}
Recently, many variance reduced stochastic alternating direction method of multipliers (ADMM) methods (e.g.\ SAG-ADMM, SDCA-ADMM and SVRG-ADMM) have made exciting progress such as linear convergence rates for strongly convex problems. However, the best known convergence rate for general convex problems is $\mathcal{O}(1/T)$ as opposed to $\mathcal{O}(1/T^2)$ of accelerated batch algorithms, where $T$ is the number of iterations. Thus, there still remains a gap in convergence rates between existing stochastic ADMM and batch algorithms. To bridge this gap, we introduce the momentum acceleration trick for batch optimization into the stochastic variance reduced gradient based ADMM (SVRG-ADMM), which leads to an accelerated (ASVRG-ADMM) method. Then we design two different momentum term update rules for strongly convex and general convex cases. We prove that ASVRG-ADMM converges linearly for strongly convex problems. Besides having a low per-iteration complexity as existing stochastic ADMM methods, ASVRG-ADMM improves the convergence rate on general convex problems from $\mathcal{O}(1/T)$ to $\mathcal{O}(1/T^2)$. Our experimental results show the effectiveness of ASVRG-ADMM.
\end{abstract}

\section{Introduction}
In this paper, we consider a class of composite convex optimization problems
\vspace{-1mm}
\begin{equation}\label{equ1}
\min_{x\in\mathbb{R}^{d_{1}}} f(x)+h(Ax),
\end{equation}
where $A\!\in\!\mathbb{R}^{d_{2}\times d_{1}}$ is a given matrix, $f(x):=\frac{1}{n}\!\sum^{n}_{i=1}\!f_{i}(x)$, each $f_{i}(x)$ is a convex function, and $h(Ax)$ is convex but possibly non-smooth. With regard to $h(\cdot)$, we are interested in a sparsity-inducing regularizer, e.g.\ $\ell_{1}$-norm, group Lasso and nuclear norm. When $A$ is an identity matrix, i.e.\ $A\!=\!I_{d_{1}}$, the above formulation (\ref{equ1}) arises in many places in machine learning, statistics, and operations research~\citep{bubeck:sgd}, such as logistic regression, Lasso and support vector machine (SVM). We mainly focus on the large sample regime. In this regime, even first-order batch methods, e.g.\ FISTA~\citep{beck:fista}, become computationally burdensome due to their per-iteration complexity of $O(nd_{1})$. As a result, stochastic gradient descent (SGD) with per-iteration complexity of $O(d_{1})$ has witnessed tremendous progress in the recent years. Especially, a number of stochastic variance reduced gradient methods such as SAG~\citep{roux:sag}, SDCA~\citep{shalev-Shwartz:sdca} and SVRG~\citep{johnson:svrg} have been proposed to successfully address the problem of high variance of the gradient estimate in ordinary SGD, resulting in a linear convergence rate (for strongly convex problems) as opposed to sub-linear rates of SGD. More recently, the Nesterov's acceleration technique~\citep{nesterov:co} was introduced in~\citep{zhu:Katyusha,hien:asmd} to further speed up the stochastic variance-reduced algorithms, which results in the best known convergence rates for both strongly convex and general convex problems. This motivates us to integrate the momentum acceleration trick into the stochastic alternating direction method of multipliers (ADMM) below.

When $A$ is a more general matrix, i.e.\ $A\!\neq\!I_{d_{1}}$, the formulation (\ref{equ1}) becomes many more complicated problems arising from machine learning, e.g.\ graph-guided fuzed Lasso~\citep{kim:flasso} and generalized Lasso~\citep{tibshirani:glasso}. To solve this class of composite optimization problems with an auxiliary variable $y\!=\!Ax$, which are the special case of the general ADMM form,
\begin{equation}\label{equ2}
\min_{x\in\mathbb{R}^{d_{1}}\!,y\in\mathbb{R}^{d_{2}}} f(x)+h(y),\;\textup{s.t.}\;Ax+By=c,
\end{equation}
the ADMM is an effective optimization tool~\citep{boyd:admm}, and has shown attractive performance in a wide range of real-world problems, such as big data classification~\citep{nie:svm}. To tackle the issue of high per-iteration complexity of batch (deterministic) ADMM (as a popular first-order optimization method), \citet{wang:oadm}, \citet{suzuki:oadmm} and \citet{ouyang:sadmm} proposed some online or stochastic ADMM algorithms. However, all these variants only achieve the convergence rate of $\mathcal{O}(1/\sqrt{T})$ for general convex problems and $\mathcal{O}(\log T/{T})$ for strongly convex problems, respectively, as compared with the $\mathcal{O}(1/T^2)$ and linear convergence rates of accelerated batch algorithms~\citep{nesterov:cp}, e.g.\ FISTA, where $T$ is the number of iterations. By now several accelerated and faster converging versions of stochastic ADMM, which are all based on variance reduction techniques, have been proposed, e.g.\ SAG-ADMM~\citep{zhong:fsadmm}, SDCA-ADMM~\citep{Suzuki:sdca} and SVRG-ADMM~\citep{zheng:fadmm}. With regard to strongly convex problems, \citet{Suzuki:sdca} and \citet{zheng:fadmm} proved that linear convergence can be obtained for the special ADMM form (i.e.\ $B\!=\!-I_{d_{2}}$ and $c\!=\!\textbf{0}$) and the general ADMM form, respectively. In SAG-ADMM and SVRG-ADMM, an $\mathcal{O}(1/T)$ convergence rate can be guaranteed for general convex problems, which implies that there still remains a gap in convergence rates between the stochastic ADMM and accelerated batch algorithms.

To bridge this gap, we integrate the momentum acceleration trick in~\citep{tseng:sco} for deterministic optimization into the stochastic variance reduction gradient (SVRG) based stochastic ADMM (SVRG-ADMM). Naturally, the proposed method has low per-iteration time complexity as existing stochastic ADMM algorithms, and does not require the storage of all gradients (or dual variables) as in SCAS-ADMM~\citep{zhao:scasadmm} and SVRG-ADMM~\citep{zheng:fadmm}, as shown in Table~\ref{tab1}. We summarize our main contributions below.

\begin{itemize}
\item We propose an accelerated variance reduced stochastic ADMM (ASVRG-ADMM) method, which integrates both the momentum acceleration trick in~\citep{tseng:sco} for batch optimization and the variance reduction technique of SVRG~\citep{johnson:svrg}.
\vspace{-1mm}
\item We prove that ASVRG-ADMM achieves a linear convergence rate for strongly convex problems, which is consistent with the best known result in SDCA-ADMM~\citep{Suzuki:sdca} and SVRG-ADMM~\citep{zheng:fadmm}.
\vspace{-1mm}
\item We also prove that ASVRG-ADMM has a convergence rate of $\mathcal{O}(1/T^{2})$ for non-strongly convex problems, which is a factor of $T$ faster than SAG-ADMM and SVRG-ADMM, whose convergence rates are $\mathcal{O}(1/T)$.
\vspace{-1mm}
\item Our experimental results further verified that our ASVRG-ADMM method has much better performance than the state-of-the-art stochastic ADMM methods.
\end{itemize}

\begin{table}[t]
\centering
\small
\caption{Comparison of convergence rates and memory requirements of some stochastic ADMM algorithms.}
\label{tab1}
\setlength{\tabcolsep}{2pt}
\renewcommand\arraystretch{1.15}
\begin{tabular}{cccc}
\hline
  & General convex   & Strongly-convex  & Space requirement\\
\hline
\!SAG-ADMM\!  & $\!\mathcal{O}(1/T)\!$ & unknown & $\!{O}(d_{1}\!d_{2}\!\!+\!nd_{1}\!)\!$\\
\!SDCA-ADMM\!  & unknown & linear rate & $\!{O}(d_{1}\!d_{2}\!\!+\!n)\!$\\
\!SCAS-ADMM\!  & $\!\mathcal{O}(1/T)\!$ & $\!\mathcal{O}(1/T)\!$ & $\!{O}(d_{1}\!d_{2})\!$\\
\!SVRG-ADMM\!  & $\!\mathcal{O}(1/T)\!$ & linear rate & $\!{O}(d_{1}\!d_{2})\!$\\
\!ASVRG-ADMM\!\!  & $\!\mathcal{O}(1/T^{2})\!$ & linear rate & $\!{O}(d_{1}\!d_{2})\!$\\
\hline
\end{tabular}
\end{table}

\section{Related Work}

Introducing $y=Ax\in\!\mathbb{R}^{d_{2}}$, problem (\ref{equ1}) becomes
\begin{equation}\label{equ3}
\min_{x\in\mathbb{R}^{d_{1}}\!,y\in\mathbb{R}^{d_{2}}} f(x)+h(y),\;\textup{s.t.}\;Ax-y=\textbf{0}.
\end{equation}
Although (\ref{equ3}) is only a special case of the general ADMM form (\ref{equ2}), when $B\!=\!-I_{d_{2}}$ and $c\!=\!\textbf{0}$, the stochastic (or online) ADMM algorithms and theoretical results in~\citep{wang:oadm,ouyang:sadmm,zhong:fsadmm,zheng:fadmm} and this paper are all for the more general problem (\ref{equ2}). To minimize (\ref{equ2}), together with the dual variable $\lambda$, the update steps of batch ADMM are
\begin{eqnarray}
&y_{k}=\arg\min_{y} h(y)+\frac{\beta}{2}\|Ax_{k-\!1}\!+\!By\!-\!c\!+\!\lambda_{k-\!1}\|^{2},\label{equ31}\\
&x_{k}=\arg\min_{x} f(x)+\frac{\beta}{2}\|Ax\!+\!By_{k}\!-\!c\!+\!\lambda_{k-\!1}\|^{2},\;\;\;\label{equ32}\\
&\lambda_{k}=\lambda_{k-\!1}+Ax_{k}+By_{k}-c,\qquad\qquad\qquad\qquad\;\:\:\label{equ33}
\end{eqnarray}
where $\beta\!>\!0$ is a penalty parameter.

To extend the batch ADMM to the online and stochastic settings, the update steps for $y_{k}$ and $\lambda_{k}$ remain unchanged. In~\citep{wang:oadm,ouyang:sadmm}, the update step of $x_{k}$ is approximated as follows:
\vspace{-1mm}
\begin{equation}\label{equ4}
\begin{split}
x_{k}=\arg\min_{x} &\:x^{T}\nabla\!f_{i_{k}}\!(x_{k-\!1})+\frac{1}{2\eta_{k}}\|x-x_{k-\!1}\|^{2}_{G}\\
&+\frac{\beta}{2}\|Ax\!+\!By_{k}\!-\!c\!+\!\lambda_{k-\!1}\|^{2},
\end{split}
\end{equation}
where we draw $i_{k}$ uniformly at random from $[n]\!:=\!\{1,\ldots,n\}$, $\eta_{k}\!\propto\!1/\sqrt{k}$ is the step-size, and $\|z\|^{2}_{G}\!=\!z^{T}Gz$ with given positive semi-definite matrix $G$, e.g.\ $G\!=\!I_{d_{1}}$ in~\citep{ouyang:sadmm}. Analogous to SGD, the stochastic ADMM variants use an unbiased estimate of the gradient at each iteration. However, all those algorithms have much slower convergence rates than their batch counterpart, as mentioned above. This barrier is mainly due to the variance introduced by the stochasticity of the gradients. Besides, to guarantee convergence, they employ a decaying sequence of step sizes $\eta_{k}$, which in turn impacts the rates.

More recently, a number of variance reduced stochastic ADMM methods (e.g.\ SAG-ADMM, SDCA-ADMM and SVRG-ADMM) have been proposed and made exciting progress such as linear convergence rates. SVRG-ADMM in~\citep{zheng:fadmm} is particularly attractive here because of its low storage requirement compared with the algorithms in~\citep{zhong:fsadmm,Suzuki:sdca}. Within each epoch of SVRG-ADMM, the full gradient $\widetilde{p}=\!\nabla\! f(\widetilde{x})$ is first computed, where $\widetilde{x}$ is the average point of the previous epoch. Then $\nabla\!f_{i_{k}}\!(x_{k-\!1})$ and $\eta_{k}$ in (\ref{equ4}) are replaced by
\begin{equation}\label{equ5}
\widetilde{\nabla}\!f_{I_{k}}\!(x_{k-\!1})=\frac{1}{|I_{k}|}\!\sum_{i_{k}\in I_{k}}\!\left(\nabla\!f_{i_{k}}\!(x_{k-\!1})-\nabla\!f_{i_{k}}\!(\widetilde{x})\right)+\widetilde{p}\,
\end{equation}
and a constant step-size $\eta$, respectively, where $I_{k}\!\subset\![n]$ is a mini-batch of size $b$ (which is a useful technique to reduce the variance). In fact, $\widetilde{\nabla}\!f_{I_{k}}\!(x_{k-\!1})$ is an unbiased estimator of the gradient $\nabla\! f(x_{k-\!1})$, i.e.\ $\mathbb{E}[\widetilde{\nabla}\!f_{I_{k}}\!(x_{k-\!1})]\!=\!\nabla\! f(x_{k-\!1})$.

\begin{algorithm}[t]
\caption{ASVRG-ADMM for strongly-convex case}
\label{alg1}
\renewcommand{\algorithmicrequire}{\textbf{Input:}}
\renewcommand{\algorithmicensure}{\textbf{Initialize:}}
\renewcommand{\algorithmicoutput}{\textbf{Output:}}
\begin{algorithmic}[1]
\REQUIRE $m$, $\eta$, $\beta>0$, $1\leq b\leq n$.\\
\ENSURE $\widetilde{x}^{0}\!=\!\widetilde{z}^{0},\widetilde{y}^{0}$, $\theta$, $\widetilde{\lambda}^{0}\!=\!-\frac{1}{\beta}(A^{T})^{\dag}\nabla\! f(\widetilde{x}_{0})$.\\
\FOR{$s=1,2,\ldots,T$}
\STATE {$x^{s}_{0}=z^{s}_{0}=\widetilde{x}^{s-\!1}$, $y^{s}_{0}=\widetilde{y}^{s-\!1}$, $\lambda^{s}_{0}=\widetilde{\lambda}^{s-\!1}$;}
\STATE {$\widetilde{p}=\nabla\!f(\widetilde{x}^{s-1})$;}
\FOR{$k=1,2,\ldots,m$}
\STATE {Choose $I_{k}\!\subseteq\! [n]\!$ of size b, uniformly at random;}
\STATE {$y^{s}_{k}\!=\!\arg\min_{y}h(y)+\frac{\beta}{2}\|Az^{s}_{k-\!1}\!+\!By-c+\!\lambda^{s}_{k-\!1}\|^{2}$;}
\STATE {$z^{s}_{k}\!=\!z^{s}_{k-\!1}\!-\!\frac{\eta\left(\widetilde{\nabla}\!f_{I_{k}}\!(x^{s}_{k-\!1})+\beta A^{T}\!(Az^{s}_{k-\!1}\!+By^{s}_{k}-c+\lambda^{s}_{k-\!1})\right)}{\gamma\theta}$;}
\STATE {$x^{s}_{k}\!=\!(1-\theta)\widetilde{x}^{s-1}+\theta z^{s}_{k}$;}
\STATE {$\lambda^{s}_{k}\!=\!\lambda^{s}_{k-1}+Az^{s}_{k}+By^{s}_{k}-c$;}
\ENDFOR
\STATE {$\widetilde{x}^{s}\!=\!\frac{1}{m}\!\sum^{m}_{k=1}\!x^{s}_{k}$,\, $\widetilde{y}^{s}\!=\!(1\!-\!\theta)\widetilde{y}^{s-\!1}\!+\!\frac{\theta}{m}\!\sum^{m}_{k=1}\!y^{s}_{k}$,}
\STATE {$\widetilde{\lambda}^{s}\!=\!-\frac{1}{\beta}(A^{T})^{\dag}\nabla f(\widetilde{x}^{s})$;}
\ENDFOR
\OUTPUT {$\widetilde{x}^{T}$, $\widetilde{y}^{T}$.}
\end{algorithmic}
\end{algorithm}

\section{Accelerated Variance Reduced Stochastic ADMM}
In this section, we design an accelerated variance reduced stochastic ADMM method for both strongly convex and general convex problems. We first make the following assumptions: Each convex $f_{i}(\cdot)$ is $L_{i}$-smooth, i.e.\ there exists a constant $L_{i}\!>\!0$ such that $\|\nabla\!f_{i}(x)\!-\!\nabla\!f_{i}(y)\|\!\leq\!L_{i}\|x\!-\!y\|$, $\forall x,y\!\in\!\mathbb{R}^{d}$, and $L\!\triangleq\!\max_{i}L_{i}$; $f(\cdot)$ is $\mu$-strongly convex, i.e.\ there is $\mu\!>\!0$ such that $f(x)\!\geq\! f(y)\!+\!\nabla\! f(y)^{T}\!(x\!-\!y)\!+\!\frac{\mu}{2}\|x\!-\!y\|^{2}$ for all $x,y\!\in\!\mathbb{R}^{d}$; The matrix $A$ has full row rank. The first two assumptions are common in the analysis of first-order optimization methods, while the last one has been used in the convergence analysis of batch ADMM~\citep{shang:rpca,nishihara:admm,deng:admm} and stochastic ADMM~\citep{zheng:fadmm}.

\subsection{The Strongly Convex Case}
In this part, we consider the case of (\ref{equ2}) when each $f_{i}(\cdot)$ is convex, $L$-smooth, and $f(\cdot)$ is $\mu$-strongly convex. Recall that this class of problems include graph-guided Logistic Regression and SVM as notable examples. To efficiently solve this class of problems, we incorporate both the momentum acceleration and variance reduction techniques into stochastic ADMM. Our algorithm is divided into $T$ epochs, and each epoch consists of $m$ stochastic updates, where $m$ is usually chosen to be $O(n)$ as in~\citep{johnson:svrg}.

Let $z$ be an important auxiliary variable, its update rule is given as follows. Similar to~\citep{zhong:fsadmm,zheng:fadmm}, we also use the inexact Uzawa method~\citep{zhang:pd} to approximate the sub-problem (\ref{equ4}), which can avoid computing the inverse of the matrix $(\frac{1}{\eta}I_{d_{1}}\!+\!\beta A^{T}\!A)$. Moreover, the momentum weight $0\!\leq\!\theta_{s}\!\leq\!1$ (the update rule for $\theta_{s}$ is provided below) is introduced into the proximal term $\frac{1}{2\eta}\|x\!-\!x_{k-\!1}\|^{2}_{G}$ similar to that of (\ref{equ4}), and then the sub-problem with respect to $z$ is formulated as follows:
\begin{equation}\label{equ6}
\begin{split}
\min_{z}&\, (z-\!z^{s}_{k-\!1}\!)^{T}\widetilde{\nabla}\!f_{I_{k}}\!(x^{s}_{k-\!1})\!+\!\frac{\theta_{\!s-\!1}}{2\eta}\!\|z-\!z^{s}_{k-\!1}\|^{2}_{G}\\
&+\frac{\beta}{2}\|Az+By^{s}_{k}-c+\lambda^{s}_{k-1}\|^{2},
\end{split}
\end{equation}
where $\widetilde{\nabla}\!f_{I_{k}}\!(x^{s}_{k-\!1})$ is defined in (\ref{equ5}), $\eta\!<\!\frac{1}{2L}$, and $G\!=\!\gamma I_{d_{1}}\!-\!\frac{\eta\beta}{\theta_{s-\!1}}A^{T}\!A$ with $\gamma\!\geq\!\gamma_{\min}\!\equiv\!\frac{\eta\beta\|A^{T}\!A\|_{2}}{\theta_{s-\!1}}\!+\!1$ to ensure that $G\succeq I$ similar to~\citep{zheng:fadmm}, where $\|\!\cdot\!\|_{2}$ is the spectral norm, i.e.\ the largest singular value of the matrix. Furthermore, the update rule for $x$ is given by
\begin{equation}\label{equ71}
x^{s}_{k}\!=\!\widetilde{x}^{s-\!1}\!+\theta_{\!s-\!1}(z^{s}_{k}\!-\widetilde{x}^{s-\!1})\!=\!(1\!-\theta_{\!s-\!1})\widetilde{x}^{s-\!1}\!+\theta_{s-\!1} z^{s}_{k},
\end{equation}
where $\theta_{\!s-\!1}(z^{s}_{k}\!-\!\widetilde{x}^{s-\!1})$ is the key momentum term (similar to those in accelerated batch methods~\citep{nesterov:co}), which helps accelerate our algorithm by using the iterate of the previous epoch, i.e.\ $\widetilde{x}^{s-\!1}$. Similar to $x^{s}_{k}$, $\widetilde{y}^{s}\!=\!(1\!-\!\theta_{s-\!1})\widetilde{y}^{s-\!1}\!+\!\frac{\theta_{s-\!1}}{m}\!\sum^{m}_{k=1}\!y^{s}_{k}$. Moreover, $\theta_{s}$ can be set to a constant $\theta$ in all epochs of our algorithm, which must satisfy $0\!\leq\!\theta\!\leq\!1\!-\!{\delta(b)}/(\alpha\!-\!1)$, where $\alpha\!=\!\frac{1}{L\eta}\!>\!1\!+\!\delta(b)$, and $\delta(b)$ is defined below. The optimal value of $\theta$ is provided in Proposition 1 below. The detailed procedure is shown in Algorithm~\ref{alg1}, where we adopt the same initialization technique for $\widetilde{\lambda}^{s}$ as in~\citep{zheng:fadmm}, and $(\cdot)^{\dag}$ is the pseudo-inverse. Note that, when $\theta\!=\!1$, ASVRG-ADMM degenerates to SVRG-ADMM in~\citep{zheng:fadmm}.

\subsection{The Non-Strongly Convex Case}
In this part, we consider general convex problems of the form (\ref{equ2}) when each $f_{i}(\cdot)$ is convex, $L$-smooth, and $h(\cdot)$ is not necessarily strongly convex (but possibly non-smooth). Different from the strongly convex case, the momentum weight $\theta_{s}$ is required to satisfy the following inequalities:
\begin{equation}\label{equ71}
\frac{1-\theta_{s}}{\theta^{2}_{s}}\leq \frac{1}{\theta^{2}_{s-\!1}}\;\,\textup{and}\;\, 0\leq\theta_{s}\leq 1-\frac{\delta(b)}{\alpha-1},
\end{equation}
where $\delta(b)\!:=\!\frac{n-b}{b(n-1)}$ is a decreasing function with respect to the mini-batch size $b$. The condition (\ref{equ71}) allows the momentum weight to decease, but not too fast, similar to the requirement on the step-size $\eta_{k}$ in classical SGD and stochastic ADMM~\citep{tseng:sgd}. Unlike batch acceleration methods, the weight must satisfy both inequalities in (\ref{equ71}).

Motivated by the momentum acceleration techniques in~\citep{tseng:sco,nesterov:co} for batch optimization, we give the update rule of the weight $\theta_{s}$ for the mini-batch case:
\begin{equation}\label{equ7}
\theta_{s}=\frac{\sqrt{\theta^{4}_{s-\!1}\!+4\theta^{2}_{s-\!1}}-\theta^{2}_{s-1}}{2}\;\,\textup{and}\;\, \theta_{0}= 1-\frac{\delta(b)}{\alpha-1}.
\end{equation}
For the special case of $b\!=\!1$, we have $\delta(1)\!=\!1$ and $\theta_{0}\!=\!1\!-\!\frac{1}{\alpha-1}$, while $b\!=\!n$ (i.e.\ batch version), $\delta(n)\!=\!0$ and $\theta_{0}\!=\!1$. Since $\{\theta_{s}\}$ is decreasing, then $\theta_{s}\!\leq\! 1\!-\!\frac{\delta(b)}{\alpha-1}$ is satisfied. The detailed procedure is shown in Algorithm~\ref{alg2}, which has many slight differences in the initialization and output of each epoch from Algorithm~\ref{alg1}. In addition, the key difference between them is the update rule for the momentum weight $\theta_{s}$. That is, $\theta_{s}$ in Algorithm~\ref{alg1} can be set to a constant, while that in Algorithm~\ref{alg2} is adaptively adjusted as in (\ref{equ7}).

\begin{algorithm}[t]
\caption{ASVRG-ADMM for general convex case}
\label{alg2}
\renewcommand{\algorithmicrequire}{\textbf{Input:}}
\renewcommand{\algorithmicensure}{\textbf{Initialize:}}
\renewcommand{\algorithmicoutput}{\textbf{Output:}}
\begin{algorithmic}[1]
\REQUIRE $m$, $\eta$, $\beta>0$, $1\leq b\leq n$.\\
\ENSURE $\widetilde{x}^{0}=\widetilde{z}^{0},\widetilde{y}^{0}$, $\widetilde{\lambda}^{0}$, $\theta_{0}=1-\frac{L\eta\delta(b)}{1-L\eta}$.\\
\FOR{$s=1,2,\ldots,T$}
\STATE {$x^{s}_{0}\!=\!(1\!-\!\theta_{s-\!1})\widetilde{x}^{s-\!1}\!+\!\theta_{s-\!1}\widetilde{z}^{s-\!1}$\!, $y^{s}_{0}\!=\!\widetilde{y}^{s-\!1}$\!, $\lambda^{s}_{0}\!=\!\widetilde{\lambda}^{s-\!1}$;}
\STATE {$\widetilde{p}=\nabla\!f(\widetilde{x}^{s-1})$, $z^{s}_{0}\!=\!\widetilde{z}^{s-\!1}$;}
\FOR{$k=1,2,\ldots,m$}
\STATE {Choose $I_{k}\!\subseteq\! [n]\!$ of size b, uniformly at random;}
\STATE {$y^{s}_{k}\!=\!\arg\min_{y}h(y)+\frac{\beta}{2}\|Az^{s}_{k-\!1}\!+\!By-c+\!\lambda^{s}_{k-\!1}\|^{2}$;}
\STATE {$z^{s}_{k}\!=\!z^{s}_{k-\!1}\!-\!\frac{\eta\left(\widetilde{\nabla}\!f_{I_{k}}\!(x^{s}_{k-\!1})+\beta A^{T}\!(Az^{s}_{k-\!1}\!+By^{s}_{k}-c+\lambda^{s}_{k-\!1})\right)}{\gamma\theta_{s-1}}$;}
\STATE {$x^{s}_{k}\!=\!(1-\theta_{s-\!1})\widetilde{x}^{s-1}+\theta_{s-\!1}z^{s}_{k}$;}
\STATE {$\lambda^{s}_{k}\!=\!\lambda^{s}_{k-1}+Az^{s}_{k}+By^{s}_{k}-c$;}
\ENDFOR
\STATE {$\widetilde{x}^{s}\!=\!\frac{1}{m}\!\sum^{m}_{k=1}\!x^{s}_{k}$,\, $\widetilde{y}^{s}\!=\!(1\!-\!\theta_{\!s-\!1})\widetilde{y}^{s-\!1}\!+\!\frac{\theta_{\!s-\!1}}{m}\!\!\sum^{m}_{k=1}\!y^{s}_{k}$,}
\STATE {$\widetilde{\lambda}^{s}\!=\!\lambda^{s}_{m}$,\, $\widetilde{z}^{s}\!=\!z^{s}_{m}$,\, $\theta_{s}\!=\!\frac{\sqrt{\theta^{4}_{s-\!1}+4\theta^{2}_{s-\!1}}-\:\!\theta^{2}_{s-\!1}}{2}$;}
\ENDFOR
\OUTPUT {$\widetilde{x}^{T}$, $\widetilde{y}^{T}$.}
\end{algorithmic}
\end{algorithm}

\section{Convergence Analysis}
This section provides the convergence analysis of our ASVRG-ADMM algorithms (i.e.\ Algorithms~\ref{alg1} and~\ref{alg2}) for strongly convex and general convex problems, respectively. Following~\citep{zheng:fadmm}, we first introduce the following function $P(x,y):= f(x)-f(x^{*})-\nabla\!f(x^{*})^{T}\!(x-x^{*})+h(y)-h(y^{*})-h'(y^{*})^{T}\!(y-y^{*})$ as a convergence criterion, where $h'(y)$ denotes the (sub)gradient of $h(\cdot)$ at $y$. Indeed, $P(x,y)\!\geq\!0$ for all $x,y\!\in\!\mathbb{R}^{d}$. In the following, we give the intermediate key results for our analysis.

\begin{lemma}\label{lemm1}
\begin{displaymath}
\begin{split}
&\mathbb{E}[\|\widetilde{\nabla}\!f_{I_{k}}\!(x^{s}_{k-\!1})-\nabla\!f(x^{s}_{k-\!1})\|^{2}]\\
\leq &2L\delta(b)\!\left[f(\widetilde{x}^{s-\!1})\!-\!f(x^{s}_{k-\!1})\!+\!(x^{s}_{k-\!1}\!-\widetilde{x}^{s-\!1})^{T}\nabla\!f(x^{s}_{k-\!1})\right],
\end{split}
\end{displaymath}
where $\delta(b)\!=\!\frac{n-b}{b(n-1)}\!\leq\!1\;\textrm{and}\; 1\leq b\leq n$.
\end{lemma}

\begin{lemma}\label{lemm2}
Using the same notation as in Lemma 1, let $(x^{*},y^{*},\lambda^{*})$ denote an optimal solution of problem (\ref{equ2}), and $\{(z^{s}_{k},x^{s}_{k}, y^{s}_{k},\lambda^{s}_{k},\widetilde{x}^{s},\widetilde{y}^{s})\}$ be the sequence generated by Algorithm~\ref{alg1} or~\ref{alg2} with $\theta_{s}\!\leq1\!-\!\frac{\delta(b)}{\alpha\!-\!1}\!$, where $\alpha\!=\!\frac{1}{L\eta}$. Then the following holds for all $k$,
\begin{equation*}
\begin{split}
&\mathbb{E}\!\left[\!P(\widetilde{x}^{s}\!,\widetilde{y}^{s})\!-\!\frac{\theta_{\!s-\!1}}{m}\!\!\sum^{m}_{k=1}\!\left((x^{*}\!-\!z^{s}_{k})^{T}\!A^{T}\!\varphi^{s}_{k}\!+\!(y^{*}\!-\!y^{s}_{k})^{T}\! B^{T}\!\varphi^{s}_{k}\right)\!\right]\!\\
\!\leq& \mathbb{E}\!\left[\frac{P(\widetilde{x}^{s-\!1}\!,\widetilde{y}^{s-\!1})}{1/(1\!-\!\theta_{\!s-\!1})}+\frac{\!\theta^{2}_{\!s-\!1}\!\left(\|x^{*}\!-z^{s}_{0}\|^{2}_{G}\!-\!\|x^{*}\!-z^{s}_{m}\|^{2}_{G}\right)}{2m\eta}\right]\\
+&\frac{\beta\theta_{\!s-\!1}}{2m}\mathbb{E}\!\left[\!\|Az^{s}_{0}\!-\!Ax^{*}\|^{2}\!-\!\|Az^{s}_{m}\!\!-\!Ax^{*}\|^{2}\!+\!\!\sum^{m}_{k=1}\!\|\lambda^{s}_{k}\!-\!\lambda^{s}_{k-\!1}\|^{2}\!\right]
\end{split}
\end{equation*}
where $\varphi^{s}_{k}=\beta(\lambda^{s}_{k}-\lambda^{*})$.
 \end{lemma}

The detailed proofs of Lemmas~\ref{lemm1} and \ref{lemm2} are provided in the Supplementary Material.

\subsection{Linear Convergence}
Our first main result is the following theorem which gives the convergence rate of Algorithm~\ref{alg1}.

\begin{theorem}\label{theo1}
Using the same notation as in Lemma~\ref{lemm2} with given $\theta\!\leq\!1\!-\!\frac{\delta(b)}{\alpha-1}$, and suppose $f(\cdot)$ is $\mu$-strongly convex and $L_{\!f}$-smooth, and $m$ is sufficiently large so that
\begin{equation}\label{equ8}
\rho\!=\!\underbrace{\frac{\theta\|\theta G\!+\!\eta\beta A^{T}\!A\|_{2}}{\eta m\mu}}_{1}\!+\underbrace{1\!-\!\theta}_{2}+\!\underbrace{\frac{L_{\!f}\theta}{\beta m\sigma_{\min}(AA^{T})}}_{3}\!<\! 1,
\end{equation}
where $\sigma_{\min}(AA^{T})$ is the smallest eigenvalue of the positive semi-definite matrix $AA^{T}$, and $G$ is defined in (\ref{equ6}). Then
\begin{equation*}
\mathbb{E}\!\left[P(\widetilde{x}^{T}\!, \widetilde{y}^{T})\right]\leq \rho^{T} P(\widetilde{x}^{0}\!, \widetilde{y}^{0}).
\end{equation*}
\end{theorem}
The proof of Theorem~\ref{theo1} is provided in the Supplementary Material. From Theorem~\ref{theo1}, one can see that ASVRG-ADMM achieves linear convergence, which is consistent with that of SVRG-ADMM, while SCAS-ADMM has only an $\mathcal{O}(1/T)$ convergence rate.

\begin{remark}
Theorem~\ref{theo1} shows that our result improves slightly upon the rate $\rho$ in~\citep{zheng:fadmm} with the same $\eta$ and $\beta$. Specifically, as shown in~(\ref{equ8}), $\rho$ consists of three components, corresponding to those of Theorem 1 in~\citep{zheng:fadmm}. In Algorithm 1, recall that here $\theta\!\leq\!1$ and $G$ is defined in (\ref{equ6}). Thus, both the first and third terms in~(\ref{equ8}) are slightly smaller than those of Theorem 1 in~\citep{zheng:fadmm}. In addition, one can set $\eta\!=\!{1}/{8L}$ (i.e.\ $\alpha\!=\!8$) and $\theta\!=\!1\!-\!{\delta(b)}/({\alpha\!-\!1})\!=\!1\!-\!{\delta(b)}/{7}$. Thus, the second term in~(\ref{equ8}) equals to ${\delta(b)}/{7}$, while that of SVRG-ADMM is approximately equal to ${4L\eta\delta(b)}/(1\!-\!4L\eta\delta(b))\!\geq\!{\delta(b)}/{2}$. In summary, the convergence bound of SVRG-ADMM can be slightly improved by ASVRG-ADMM.
\end{remark}

\subsection{Selecting Scheme of $\theta$}

The rate $\rho$ in~(\ref{equ8}) of Theorem~\ref{theo1} can be expressed as the function with respect to the parameters $\theta$ and $\beta$ with given $m,\eta, L_{f}, L, A, \mu$. Similar to~\citep{nishihara:admm,zheng:fadmm}, one can obtain the optimal parameter $\beta^{*}\!=\!\sqrt{{L_{\!f}\mu}/(\sigma_{\min}(AA^{T})\|A^{T}\!A\|_{2})}$, which produces a smaller rate $\rho$. In addition, as shown in~(\ref{equ8}), all the three terms are with respect to the weight $\theta$. Therefore, we give the following selecting scheme for $\theta$.

\begin{proposition}
Given $\kappa_{\!f}\!=\!L_{\!f}/\mu,\beta^{*},\kappa\!=\!L/\mu, b, A$, and let $\omega\!=\!\|A^{T}\!A\|_{2}/\sigma_{\min}(AA^{T})$, we set $m\!>\!2\kappa\!+\!2\sqrt{\kappa_{\!f}\omega}$ and $\eta\!=\!{1}/{(L\alpha)}$, where $\alpha\!=\!\frac{m-2\sqrt{\kappa_{\!f} \omega}}{2\kappa}\!+\!\delta(b)\!+\!1$. Then the optimal $\theta^{*}$ of Algorithm 1 is given by
\begin{equation*}
\theta^{*}=\frac{m-2\sqrt{\kappa_{\!f}\omega}}{m-2\sqrt{\kappa_{\!f}\omega}+2\kappa(\delta(b)+1)}.
\end{equation*}
\end{proposition}

The proof of Proposition 1 is provided in the Supplementary Material.

\begin{figure*}[t]
\centering
\includegraphics[width=0.493\columnwidth]{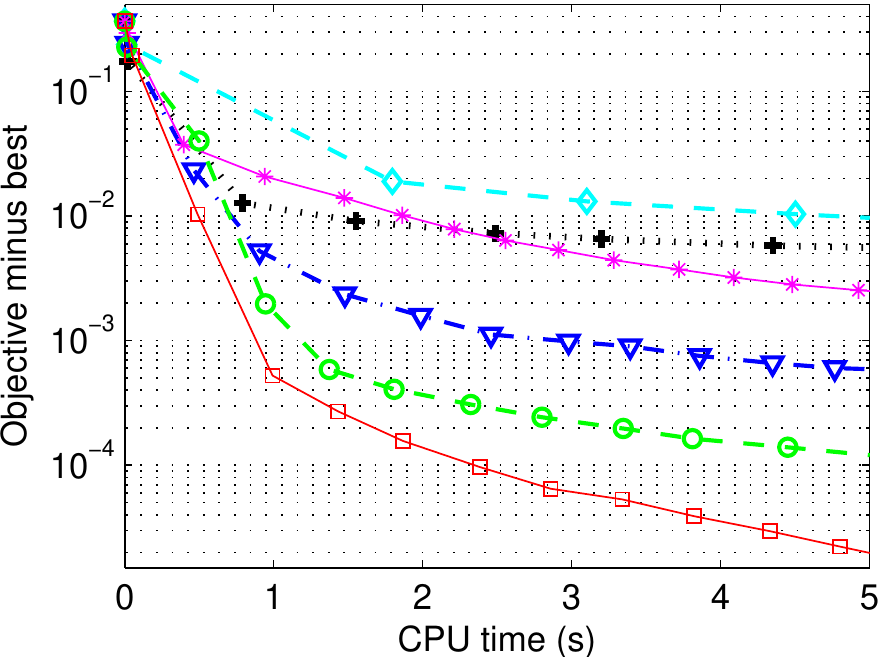}\;
\includegraphics[width=0.493\columnwidth]{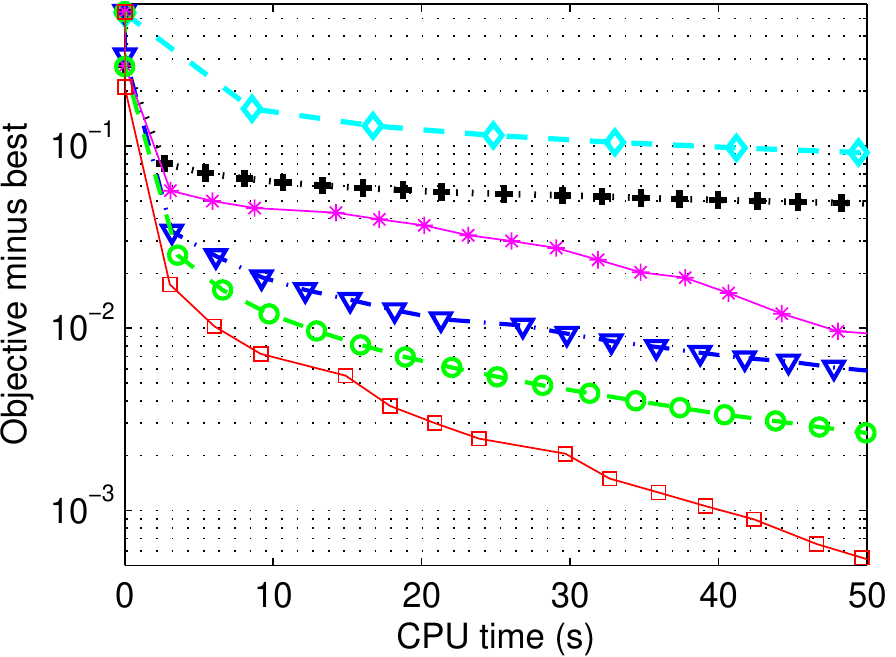}\;
\includegraphics[width=0.493\columnwidth]{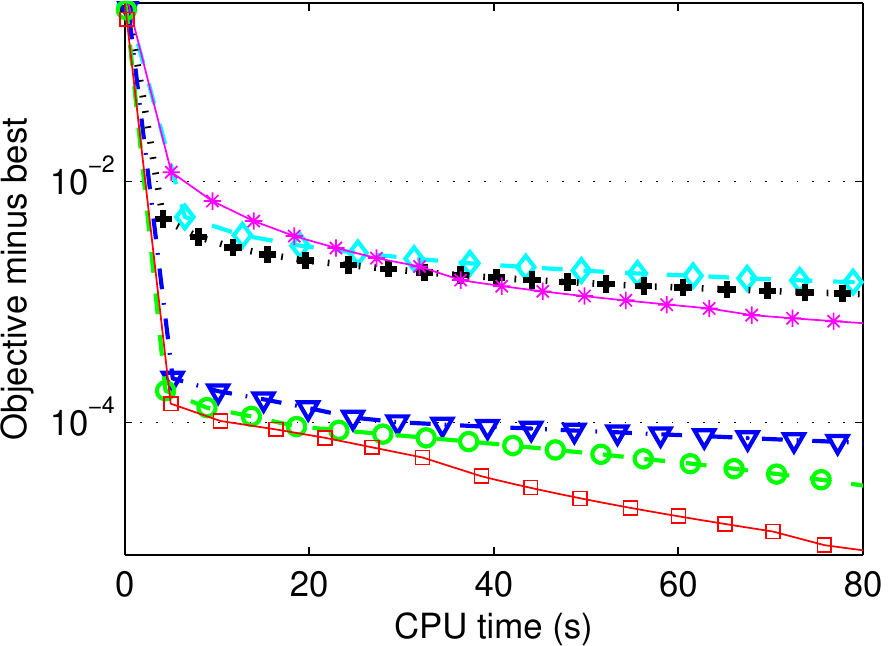}\;
\includegraphics[width=0.493\columnwidth]{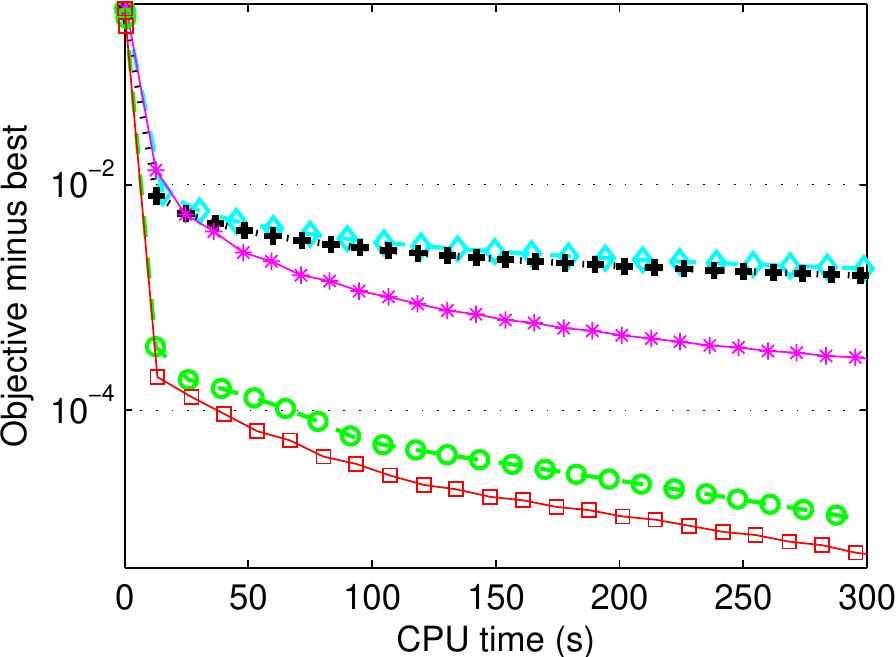}
\subfigure[a9a]{\includegraphics[width=0.493\columnwidth]{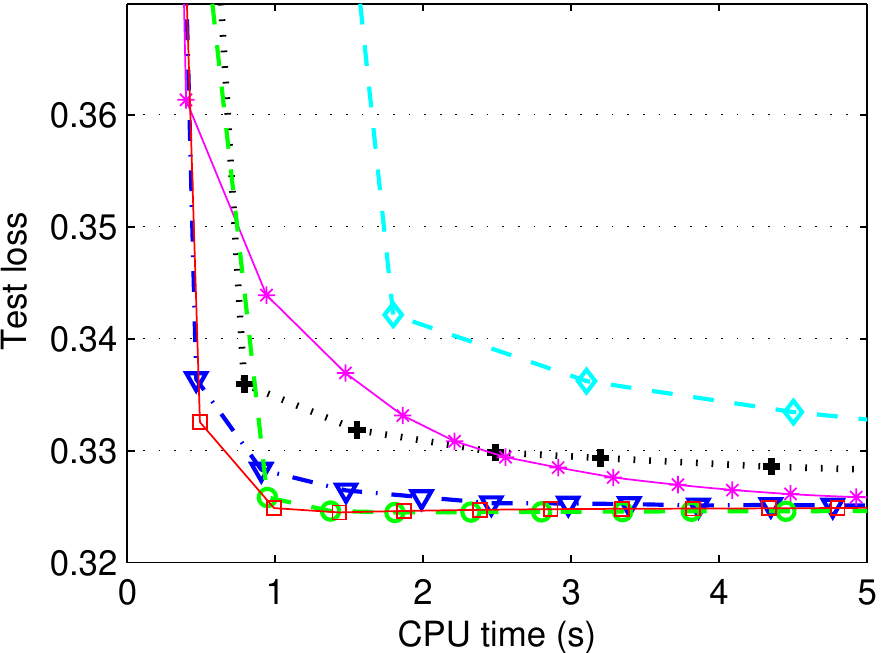}}\;
\subfigure[w8a]{\includegraphics[width=0.493\columnwidth]{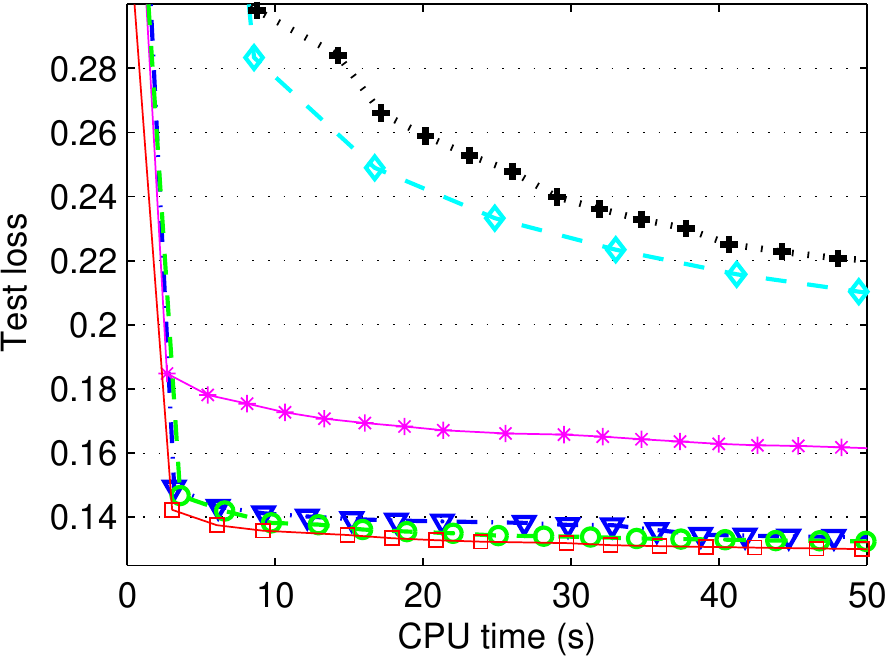}}\;
\subfigure[SUSY]{\includegraphics[width=0.493\columnwidth]{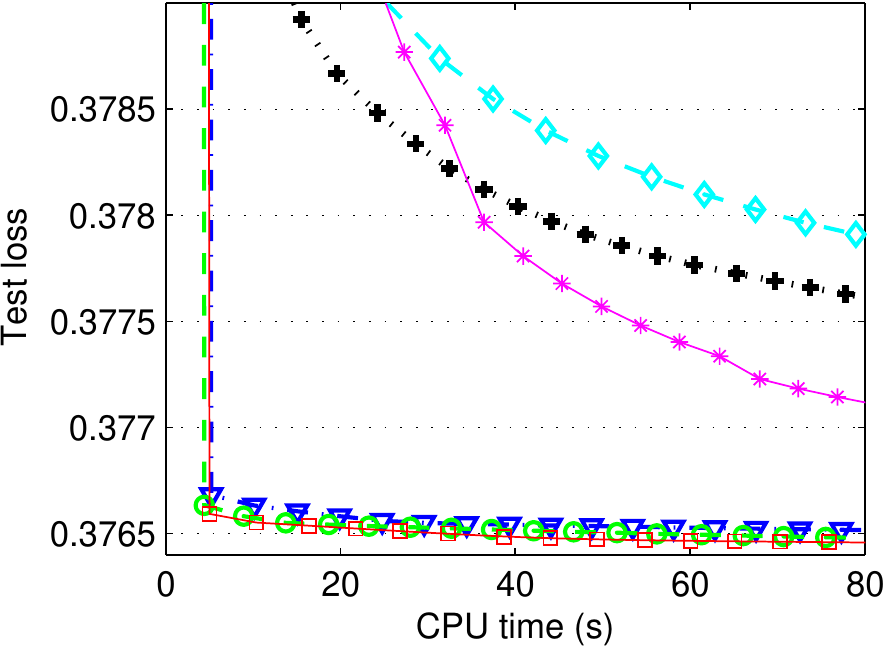}}\;
\subfigure[HIGGS]{\includegraphics[width=0.493\columnwidth]{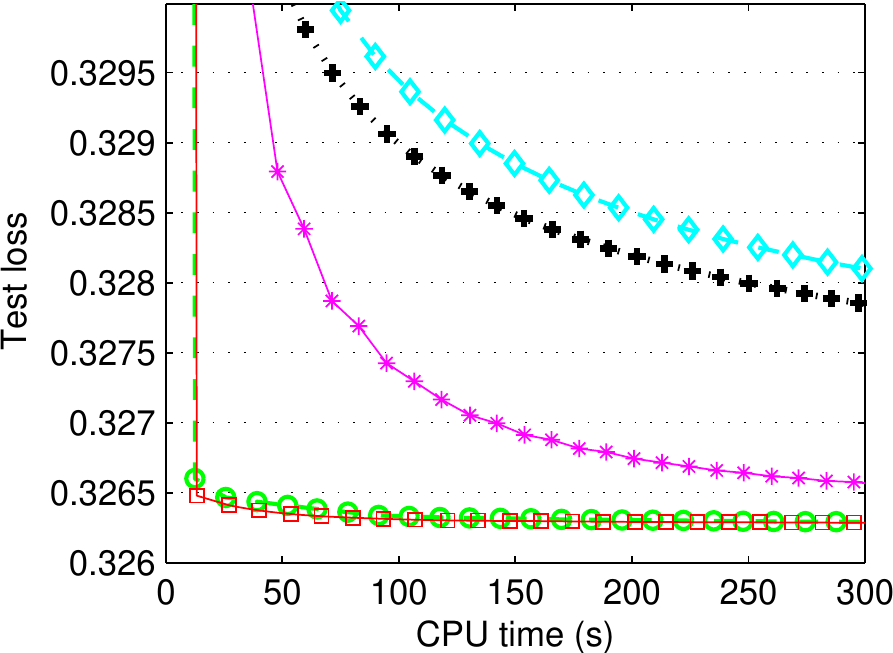}}
\includegraphics[width=0.3\columnwidth]{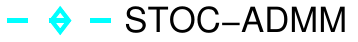}\;
\includegraphics[width=0.285\columnwidth]{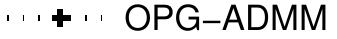}\;
\includegraphics[width=0.285\columnwidth]{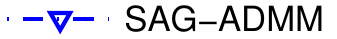}\;
\includegraphics[width=0.3\columnwidth]{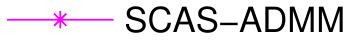}\;
\includegraphics[width=0.3\columnwidth]{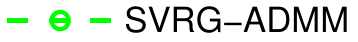}\;
\includegraphics[width=0.3\columnwidth]{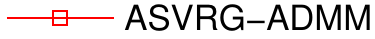}
\vspace{-2mm}
\caption{Comparison of different stochastic ADMM methods for graph-guided fuzed Lasso problems on the four data sets. The $x$-axis represents the objective value minus the minimum (top) or testing loss (bottom), and the $y$-axis corresponds to the running time (seconds).}
\label{fig_sim1}
\end{figure*}

\subsection{Convergence Rate of $\mathcal{O}(1/T^{2})$}
We first assume that $z\!\in\!\mathcal{Z}$, where $\mathcal{Z}$ is a convex compact set with diameter $D_{\mathcal{Z}}\!=\!\sup_{z_{1}\!,z_{2}\in \mathcal{Z}}\|z_{1}\!-\!z_{2}\|$, and the dual variable $\lambda$ is also bounded with $D_{\lambda}\!=\!\sup_{\lambda_{1}\!,\lambda_{2}}\|\lambda_{1}\!-\!\lambda_{2}\|$. For Algorithm 2, we give the following result.

\begin{theorem}\label{theo2}
Using the same notation as in Lemma~\ref{lemm2} with $\theta_{0}\!= \!1\!-\!\frac{\delta(b)}{\alpha-1}$, then we have
\begin{equation}\label{equ141}
\begin{split}
&\mathbb{E}\!\left[P(\widetilde{x}^{T}\!,\widetilde{y}^{T})\!+\gamma\|A\widetilde{x}^{T}\!+B\widetilde{y}^{T}\!-c\|\right]\\
\leq&\, \frac{4(\alpha\!-\!1)\delta(b)\left(P(\widetilde{x}^{0}\!,\widetilde{y}^{0})+ \gamma\|A\widetilde{x}^{0}\!+B\widetilde{y}^{0}\!-c\|\right)}{(\alpha-1-\delta(b))^{2}(T\!+\!1)^{2}}\\
&+\frac{2L\alpha\|x^{*}\!-\widetilde{x}^{0}\|^{2}_{G}}{m(T+1)^{2}}+\frac{2\beta\left(\|A^{T}\!A\|_{2}D^{2}_{\mathcal{Z}}+4D^{2}_{\lambda}\right)}{m(T\!+\!1)}.
\end{split}
\end{equation}
\end{theorem}

The proof of Theorem~\ref{theo2} is provided in the Supplementary Material. Theorem 2 shows that the convergence bound consists of the three components, which converge as $\mathcal{O}(1/T^{2})$, $\mathcal{O}(1/mT^{2})$ and $\mathcal{O}(1/mT)$, respectively, while the three components of SVRG-ADMM converge as $\mathcal{O}(1/T)$, $\mathcal{O}(1/mT)$ and $\mathcal{O}(1/mT)$. Clearly, ASVRG-ADMM achieves the convergence rate of $\mathcal{O}(1/T^{2})$ as opposed to $\mathcal{O}(1/T)$ of SVRG-ADMM and SAG-ADMM ($m\!\gg\!T$). All the components in the convergence bound of SCAS-ADMM converge as $\mathcal{O}(1/T)$. Thus, it is clear from this comparison that ASVRG-ADMM is a factor of $T$ faster than SAG-ADMM, SVRG-ADMM and SCAS-ADMM.

\subsection{Connections to Related Work}
Our algorithms and convergence results can be extended to the following settings. When the mini-batch size $b\!=\!n$ and $m\!=\!1$, then $\delta(n)\!=\!0$, that is, the first term of (\ref{equ141}) vanishes, and ASVRG-ADMM degenerates to the batch version. Its convergence rate becomes $\mathcal{O}(D^{2}_{x^{*}}\!/{(T\!+\!1)}^{2}\!+\!{D^{2}_{\mathcal{Z}}}/{(T\!+\!1)}\!+\!{D^{2}_{\lambda}}/{(T\!+\!1)})$ (which is consistent with the optimal result for accelerated deterministic ADMM methods~\citep{Goldstein:admm,lu:admm}), where $D_{x^{*}}\!=\!\|x^{*}\!-\!\widetilde{x}^{0}\|_{G}$. Many empirical risk minimization problems can be viewed as the special case of (\ref{equ1}) when $A\!=\!I$. Thus, our method can be extended to solve them, and has an $\mathcal{O}({1}/{T^{2}}\!+\!{1}/{(mT^{2})})$ rate, which is consistent with the best known result as in~\citep{zhu:Katyusha,hien:asmd}.

\section{Experiments}
In this section, we use our ASVRG-ADMM method to solve the general convex graph-guided fuzed Lasso, strongly convex graph-guided logistic regression and graph-guided SVM problems. We compare ASVRG-ADMM with the following state-of-the-art methods: STOC-ADMM~\citep{ouyang:sadmm}, OPG-ADMM~\citep{suzuki:oadmm}, SAG-ADMM~\citep{zhong:fsadmm}, and SCAS-ADMM~\citep{zhao:scasadmm} and SVRG-ADMM~\citep{zheng:fadmm}. All methods were implemented in MATLAB, and the experiments were performed on a PC with an Intel i5-2400 CPU and 16GB RAM.

\subsection{Graph-Guided Fused Lasso}
We first evaluate the empirical performance of the proposed method for solving the graph-guided fuzed Lasso problem:
\begin{equation}\label{equ51}
\min_{x}\frac{1}{n}\sum^{n}_{i=1}\ell_{i}(x)+\lambda_{1}\|Ax\|_{1},
\end{equation}
where $\ell_{i}$ is the logistic loss function on the feature-label pair $(a_{i},b_{i})$, i.e., $\log(1\!+\!\exp(-b_{i}a^{T}_{i}x))$, and $\lambda_{1}\!\geq\!0$ is the regularization
parameter. Here, we set $A=[G;I]$ as in~\citep{ouyang:sadmm,zhong:fsadmm,azadi:sadmm,zheng:fadmm}, where $G$ is the sparsity pattern of the graph obtained by sparse inverse covariance selection~\citep{banerjee:mle}. We used four publicly available data sets{\footnote{\url{http://www.csie.ntu.edu.tw/~cjlin/libsvmtools/datasets/}}} in our experiments, as listed in Table~\ref{tab2}. Note that except STOC-ADMM, all the other algorithms adopted the linearization of the penalty term $\frac{\beta}{2}\|Ax\!-\!y\!+\!z\|^{2}$ to avoid the inversion of $\frac{1}{\eta_{k}}\!I_{d_{1}}\!+\!\beta A^{T}\!A$ at each iteration, which can be computationally expensive for large matrices. The parameters of ASVRG-ADMM are set as follows: $m\!=\!2n/b$ and $\gamma\!=\!1$ as in~\citep{zhong:fsadmm,zheng:fadmm}, as well as $\eta$ and $\beta$.

\begin{table}[!th]
\centering
\small
\caption{Summary of data sets and regularization parameters used in our experiments.}
\label{tab2}
\begin{tabular}{lccccc}
\hline
Data sets   &\!$\sharp$ training   &$\sharp$ test  & \!\!$\sharp$ mini-batch\!\!  &$\lambda_{1}$ &$\lambda_{2}$\\
\hline
 a9a        & 16,281         & 16,280                       & 20         & 1e-5   & 1e-2\\
 w8a        & 32,350         & 32,350                       & 20         & 1e-5   & 1e-2\\
 SUSY       & \!\!3,500,000\!      & \!1,500,000\!          & 100        & 1e-5   & 1e-2\\
 HIGGS      & \!\!7,700,000\!      & \!3,300,000\!          & 150        & 1e-5   & 1e-2\\
\hline
\end{tabular}
\end{table}

Figure~\ref{fig_sim1} shows the training error (i.e.\ the training objective value minus the minimum) and testing loss of all the algorithms for the general convex problem on the four data sets. SAG-ADMM could not generate experimental results on the HIGGS data set because it ran out of memory. These figures clearly indicate that the variance reduced stochastic ADMM algorithms (including SAG-ADMM, SCAS-ADMM, SVRG-ADMM and ASVRG-ADMM) converge much faster than those without variance reduction techniques, e.g.\ STOC-ADMM and OPG-ADMM. Notably, ASVRG-ADMM consistently outperforms all other algorithms in terms of the convergence rate under all settings, which empirically verifies our theoretical result that ASVRG-ADMM has a faster convergence rate of $\mathcal{O}(1/T^{2})$, as opposed to the best known rate of $\mathcal{O}(1/T)$.

\begin{figure}[t]
\centering
\includegraphics[width=0.486\columnwidth]{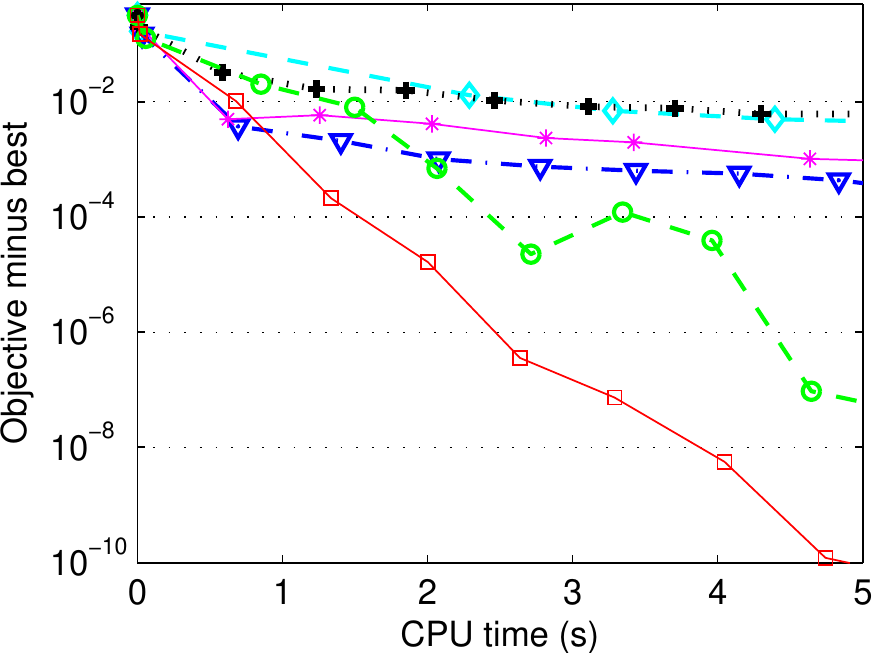}\;
\includegraphics[width=0.486\columnwidth]{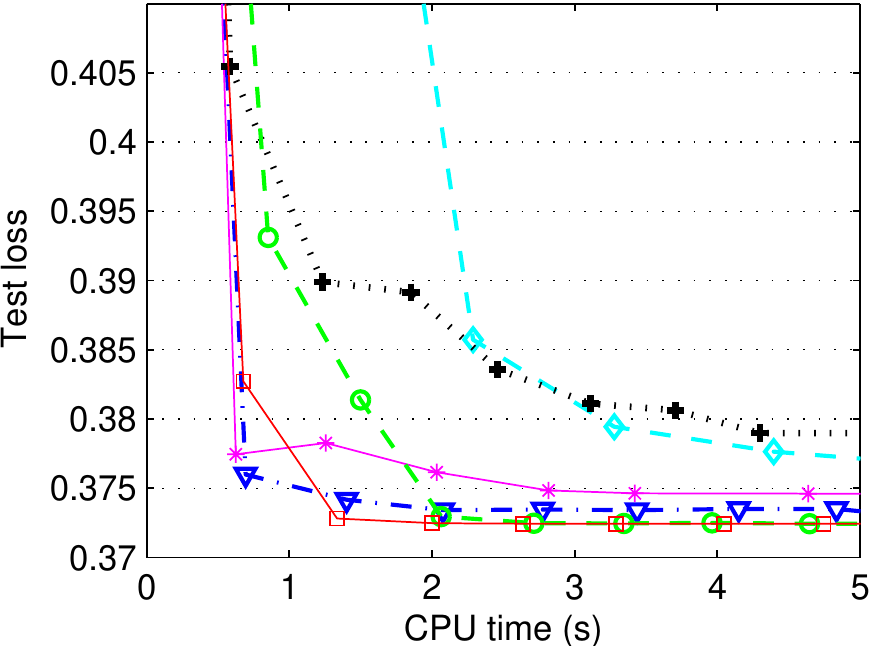}
\includegraphics[width=0.486\columnwidth]{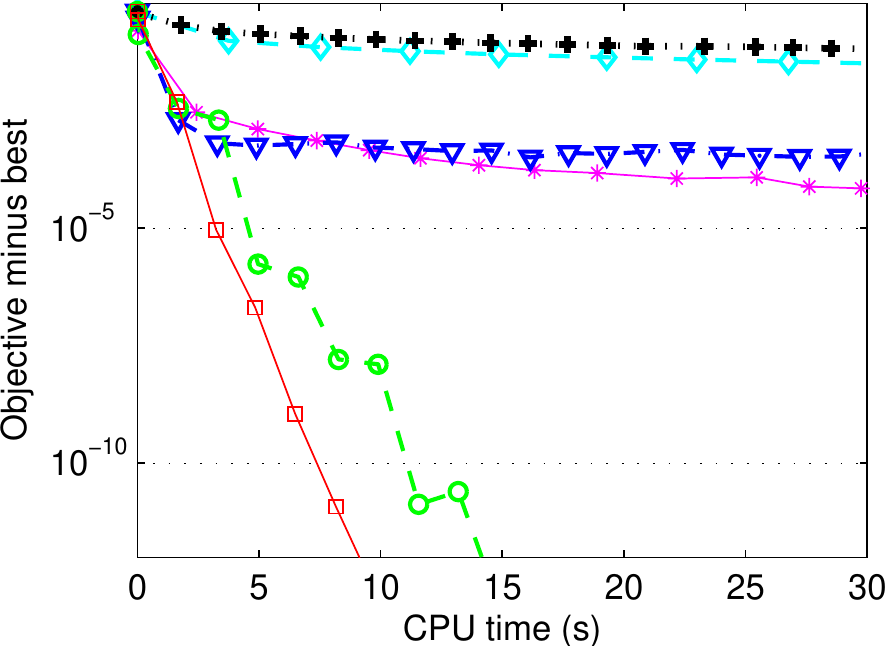}\;
\includegraphics[width=0.486\columnwidth]{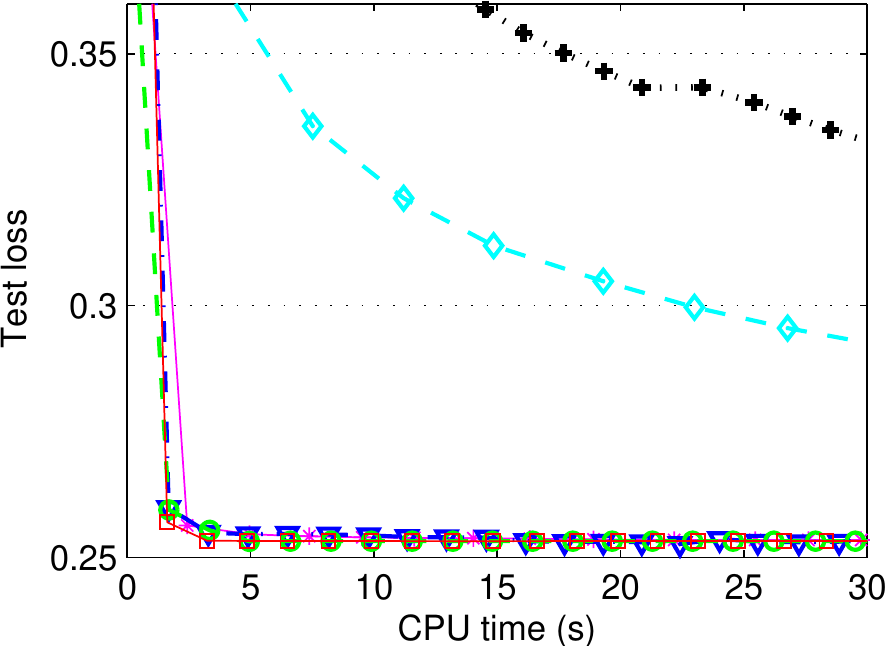}
\includegraphics[width=0.285\columnwidth]{line1}\;
\includegraphics[width=0.271\columnwidth]{line2}\;\:
\includegraphics[width=0.271\columnwidth]{line3}\;\,\\
\includegraphics[width=0.285\columnwidth]{line4}\;
\includegraphics[width=0.285\columnwidth]{line5}\;
\includegraphics[width=0.285\columnwidth]{line6}
\vspace{-2mm}
\caption{Comparison of different stochastic ADMM methods for graph-guided logistic regression problems on the two data sets: a9a (top) and w8a (bottom).}
\label{fig_sim2}
\end{figure}

\subsection{Graph-Guided Logistic Regression}
We further discuss the performance of ASVRG-ADMM for solving the strongly convex graph-guided logistic regression problem~\citep{ouyang:sadmm,zhong:stoc}:
\begin{equation}\label{equ52}
\min_{x}\frac{1}{n}\sum^{n}_{i=1}\!\left(\ell_{i}(x)+\frac{\lambda_{2}}{2} \|x\|^{2}_{2}\right)+\lambda_{1} \|Ax\|_{1}.
\end{equation}
Due to limited space and similar experimental phenomena on the four data sets, we only report the experimental results on the a9a and w8a data sets in Figure~\ref{fig_sim2}, from which we observe that SVRG-ADMM and ASVRG-ADMM achieve comparable performance, and they significantly outperform the other methods in terms of the convergence rate, which is consistent with their linear (geometric) convergence guarantees. Moreover, ASVRG-ADMM converges slightly faster than SVRG-ADMM, which shows the effectiveness of the momentum trick to accelerate variance reduced stochastic ADMM, as we expected.

\subsection{Graph-Guided SVM}
Finally, we evaluate the performance of ASVRG-ADMM for solving the graph-guided SVM problem,
\begin{equation}\label{equ53}
\min_{x}\frac{1}{n}\sum^{n}_{i=1}\!\left([1-b_{i}a^{T}_{i}x]_{+}+\frac{\lambda_{2}}{2}\|x\|^{2}_{2}\right)+\lambda_{1}\|Ax\|_{1},
\end{equation}
where $[x]_{+}\!=\!\max(0,x)$ is the non-smooth hinge loss. To effectively solve problem~(\ref{equ53}), we used the smooth Huberized hinge loss in~\citep{rosset:svm} to approximate the hinge loss. For the 20newsgroups dataset{\footnote{\url{http://www.cs.nyu.edu/~roweis/data.html}}}, we randomly divide it into 80\% training set and 20\% test set. Following~\citep{ouyang:sadmm}, we set $\lambda_{1}\!=\!\lambda_{2}\!=\!10^{-5}$, and use the one-vs-rest scheme for the multi-class classification.

\begin{figure}[t]
\centering
\includegraphics[width=0.486\columnwidth]{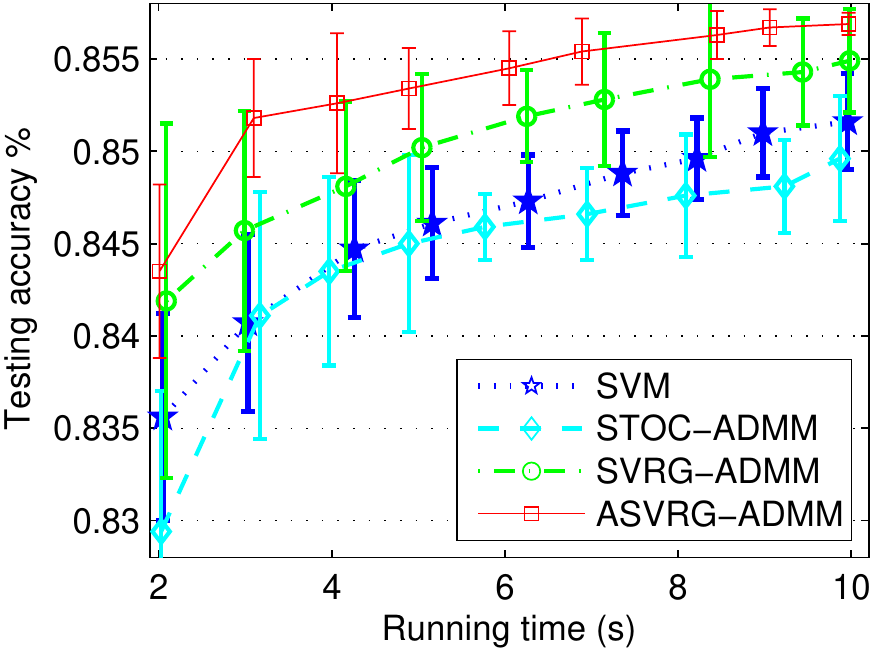}\;
\includegraphics[width=0.486\columnwidth]{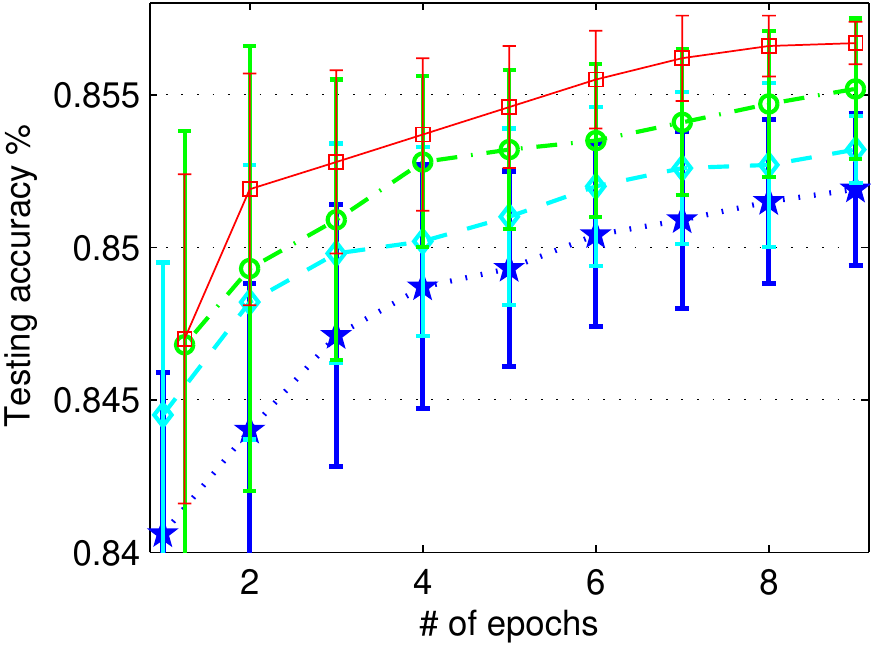}
\vspace{-6mm}
\caption{Comparison of accuracies multi-class classification on the 20newsgroups data set: accuracy v.s.\ running time (left) or number of epochs (right).}
\label{fig_sim3}
\end{figure}

Figure~\ref{fig_sim3} shows the average prediction accuracies and standard deviations of testing accuracies over 10 different runs. Since STOC-ADMM, OPG-ADMM, SAG-ADMM and SCAS-ADMM consistently perform worse than SVRG-ADMM and ASVRG-ADMM in all settings, we only report the results of STOC-ADMM. We observe that SVRG-ADMM and ASVRG-ADMM consistently outperform the classical SVM and STOC-ADMM. Moreover, ASVRG-ADMM performs much better than the other methods in all settings, which again verifies the effectiveness of our ASVRG-ADMM method.

\section{Conclusions}
In this paper, we proposed an accelerated stochastic variance reduced ADMM (ASVRG-ADMM) method, in which we combined both the momentum acceleration trick for batch optimization and the variance reduction technique. We designed two different momentum term update rules for strongly convex and general convex cases, respectively. Moreover, we also theoretically analyzed the convergence properties of ASVRG-ADMM, from which it is clear that ASVRG-ADMM achieves linear convergence and $\mathcal{O}(1/T^{2})$ rates for both cases. Especially, ASVRG-ADMM is at least a factor of $T$ faster than existing stochastic ADMM methods for general convex problems.

\section{Acknowledgements}
We thank the reviewers for their valuable comments. The authors are supported by the Hong Kong GRF 2150851 and 2150895, and Grants 3132964 and 3132821 funded by the Research Committee of CUHK.

\bibliographystyle{aaai}
\bibliography{aaai2017}

%
\newpage
\onecolumn

\setlength{\floatsep}{3pt}
\setlength{\abovecaptionskip}{2pt}
\setlength{\belowcaptionskip}{3pt}

\section{Supplementary Materials for ``Accelerated Variance Reduced Stochastic ADMM"}

\maketitle

\setcounter{page}{7}
\setcounter{property}{0}
\setcounter{lemma}{2}
\setcounter{equation}{17}

In this supplementary material, we give the detailed proofs for two important lemmas (i.e., Lemmas 1 and 2), two key theorems (i.e., Theorems 1 and 2) and a proposition (i.e., Proposition 1).

\section{Proof of Lemma 1:}
Our convergence analysis will use a bound on the variance term $\mathbb{E}[\|\widetilde{\nabla}\!f_{I_{k}}\!(x^{s}_{k-\!1})-\nabla\!f(x^{s}_{k-\!1})\|^{2}]$, as shown in Lemma 1. Before giving the proof of Lemma 1, we first give the following lemma.
\begin{lemma}\label{lemm1}

Since each $f_{j}(x)$ is convex, $L_{j}$-smooth ($j\!=\!1,\ldots,n$), then the following holds
\begin{equation}\label{equ11}
\begin{split}
&\|\nabla f_{i_{k}}\!(x^{s}_{k-1})-\nabla f_{i_{k}}\!(\widetilde{x}^{s-\!1})\|^{2}\\
\leq& 2L\!\left[f_{i_{k}}\!(\widetilde{x}^{s-\!1})\!-\!\!f_{i_{k}}\!(x^{s}_{k-\!1})\!+\!\langle \nabla\! f_{i_{k}}\!(x^{s}_{k-\!1}),\, x^{s}_{k-\!1}\!\!-\!\widetilde{x}^{s-\!1}\rangle\right],
\end{split}
\end{equation}
where $i_{k}\in[n]$, and $L:=\max_{j}L_{j}$.
\end{lemma}

\begin{proof}
This result follows immediately from Theorem 2.1.5 in~\citep{nesterov:co}.
\end{proof}
\vspace{3mm}

\textbf{ Proof of Lemma 1:}
\begin{proof}
$\widetilde{\nabla}\! f_{i_{k}}\!(x^{s}_{k-\!1})\!=\!\nabla\! f_{i_{k}}\!(x^{s}_{k-\!1})\!-\!\nabla\! f_{i_{k}}\!(\widetilde{x}^{s-\!1})\!+\!\nabla\! f(\widetilde{x}^{s-\!1})$. Taking expectation over the random choice of $i_{k}$, we have
\begin{equation}\label{equ12}
\begin{split}
&\mathbb{E}\!\left[\left\|\widetilde{\nabla} f_{i_{k}}\!(x^{s}_{k-1})-\nabla f(x^{s}_{k-1})\right\|^{2}\right]\\
=&\,\mathbb{E}\!\left[\left\|(\nabla\! f(\widetilde{x}^{s-\!1})\!-\!\nabla\! f(x^{s}_{k-\!1}))\!-\!(\nabla\! f_{i_{k}}\!(\widetilde{x}^{s-\!1})\!-\!\nabla\! f_{i_{k}}\!(x^{s}_{k-\!1}))\right\|^{2}\right]\\
\leq&\, \mathbb{E}\!\left[\left\|\nabla f_{i_{k}}\!(x^{s}_{k-1})-\nabla f_{i_{k}}\!(\widetilde{x}^{s-1})\right\|^{2}\right]\\
\leq&\, 2L\;\!\mathbb{E}\!\left[f_{i_{k}}\!(\widetilde{x}^{s-\!1})-f_{i_{k}}\!(x^{s}_{k-\!1})+\langle \nabla\! f_{i_{k}}\!(x^{s}_{k-\!1}),\;x^{s}_{k-\!1}\!-\widetilde{x}^{s-\!1}\rangle\right]\\
=&\,2L\!\left(f(\widetilde{x}^{s-1})-f(x^{s}_{k-1})+\langle \nabla f(x^{s}_{k-1}),\; x^{s}_{k-1}-\widetilde{x}^{s-1}\rangle\right),
\end{split}
\end{equation}
where the first inequality follows from the fact that $\mathbb{E}[\|\mathbb{E}[x]-x\|^{2}]=\mathbb{E}[\|x\|^{2}]-\|\mathbb{E}[x]\|^{2}$, and the second inequality is due to Lemma 3 given above. Note that the similar result for (\ref{equ12}) was also proved in~\citep{zhu:Katyusha} (see Lemma 3.4 in~\citep{zhu:Katyusha}). Next, we extend the result to the mini-batch setting.

Let $b$ be the size of mini-batch $I_{k}$. We prove the result of Lemma 1 for the mini-batch case, i.e.\ $b\geq2$.
\begin{displaymath}
\begin{split}
&\mathbb{E}\!\left[\left\|\widetilde{\nabla}\! f_{I_{k}}\!(x^{s}_{k-1})-\nabla\! f(x^{s}_{k-1})\right\|^{2}\right]\\
\!\!\!\!\!\!=&\mathbb{E}\!\!\left[\left\|\frac{1}{b}\!\!\sum_{i\in{I}_{k}}\!\!\!\left(\nabla\! f_{i}(x^{s}_{k-\!1})\!-\!\!\nabla\! f_{i}(\widetilde{x}^{s-\!1})\right)\!+\!\!\nabla\! f(\widetilde{x}^{s-\!1})\!-\!\!\nabla\! f(x^{s}_{k-\!1})\right\|^{2}\right]\\
\!\!\!\!\!\!=&\frac{n\!-\!b}{b(n\!-\!1)}\mathbb{E}\!\!\left[\left\|\nabla\! f_{i}(x^{s}_{k-\!1})\!-\!\!\nabla\! f(x^{s}_{k-\!1})\!-\!\!\nabla\! f_{i}(\widetilde{x}^{s-\!1})\!+\!\!\nabla\! f(\widetilde{x}^{s-\!1})\right\|^{2}\right]\\
\!\!\!\!\!\!\leq&\frac{2L(n\!-\!b)}{b(n\!-\!1)}\left(f(\widetilde{x}^{s-\!1})\!-\!f(x^{s}_{k-\!1})\!+\!\left\langle \nabla\! f(x^{s}_{k-\!1}),\;x^{s}_{k-\!1}\!-\!\widetilde{x}^{s-\!1}\right\rangle\right),
\end{split}
\end{displaymath}
where the second equality follows from Lemma 4 in~\citep{koneeny:mini}, and the inequality holds due to the result in (\ref{equ12}).
\end{proof}
\vspace{3mm}

\section{Proof of Lemma 2:}
Before proving the key Lemma 2, we first give the following a property~\citep{baldassarre:prox,lan:sgd}, which is useful for the convergence analysis of ASVRG-ADMM.
\vspace{1mm}

\begin{property}\label{proper1}
Given any $w_{1}, w_{2}, w_{3}, w_{4}\in \mathbb{R}^{d}$, then we have
\begin{displaymath}
\begin{split}
&\langle w_{1}-w_{2}, \;w_{1}-w_{3}\rangle=\frac{1}{2}\!\left(\|w_{1}-w_{2}\|^{2}+\|w_{1}-w_{3}\|^{2}-\|w_{2}-w_{3}\|^{2}\right),
\end{split}
\end{displaymath}
\vspace{1mm}
\begin{displaymath}
\begin{split}
&\langle w_{1}-w_{2}, \;w_{3}-w_{4}\rangle=\frac{1}{2}\!\left(\|w_{1}\!-\!w_{4}\|^{2}\!-\!\|w_{1}\!-\!w_{3}\|^{2}\!+\!\|w_{2}\!-\!w_{3}\|^{2}\!-\!\|w_{2}\!-\!w_{4}\|^{2}\right).
\end{split}
\end{displaymath}
\end{property}

\onecolumn

In order to prove Lemma 2, we first give and prove the following two key lemmas.

\begin{lemma}\label{lemm2}
Since $\eta=\frac{1}{L\alpha}$, $1-\theta_{s-1}\geq \frac{\delta(b)}{\alpha-1}$, and $\delta(b)=\frac{n-b}{b(n-1)}$, then we have
\begin{equation*}
\begin{split}
&\mathbb{E}\!\left[f(\widetilde{x}^{s})-f(x^{*})-\left\langle\nabla f(x^{*}),\;\widetilde{x}^{s}-x^{*}\right\rangle-\frac{\theta_{s-1}}{m}\!\sum^{m}_{k=1}\!\left\langle A^{T}\varphi^{s}_{k},\;x^{*}-z^{s}_{k}\right\rangle\right]\\
\leq& \mathbb{E}\!\left[(1\!-\!\theta_{s-1})\left(f(\widetilde{x}^{s-\!1})-f(x^{*})-\langle\nabla f(x^{*}),\;\widetilde{x}^{s-\!1}-x^{*}\rangle\right)+\frac{L\alpha\theta^{2}_{s-\!1}}{2m}\!\left(\|x^{*}-z^{s}_{0}\|^{2}-\|x^{*}-z^{s}_{m}\|^{2}\right)\right].
\end{split}
\end{equation*}
\end{lemma}
\vspace{3mm}

\begin{proof}
Let $g_{{k}}=\frac{1}{b}\!\sum_{i_{k}\in{I}_{k}}\!\left(\nabla f_{i_{k}}\!(x^{s}_{k-1})-\nabla f_{i_{k}}\!(\widetilde{x}^{s-1})\right)+\nabla f(\widetilde{x}^{s-1})$. Since the function $f$ is convex, differentiable  with an $L_{f}$-Lipschitz-continuous gradient, where $L_{f}\leq L=\max_{j=1,\ldots,n}L_{j}$, then
\begin{equation}\label{equ71}
\begin{split}
f(x^{s}_{k})\leq\,& f(x^{s}_{k-1})+\left\langle\nabla f(x^{s}_{k-1}),\,x^{s}_{k}-x^{s}_{k-1}\right\rangle+\frac{L\alpha}{2}\!\left\|x^{s}_{k}-x^{s}_{k-1}\right\|^{2}-\frac{L(\alpha\!-\!1)}{2}\!\left\|x^{s}_{k}-x^{s}_{k-1}\right\|^{2}\\
=\,& f(x^{s}_{k-1})+\left\langle g_{k},\,x^{s}_{k}-x^{s}_{k-1}\right\rangle+\frac{L\alpha}{2}\|x^{s}_{k}-x^{s}_{k-1}\|^2\\
&+\left\langle\nabla f(x^{s}_{k-1})-g_{k},\,x^{s}_{k}-x^{s}_{k-1}\right\rangle-\frac{L(\alpha\!-\!1)}{2}\|x^{s}_{k}-x^{s}_{k-1}\|^{2}.
\end{split}
\end{equation}

Using Lemma 1, then we get
\begin{equation}\label{equ72}
\begin{split}
&\mathbb{E}\!\left[\left\langle\nabla\! f(x^{s}_{k-1})-g_k,\,x^{s}_{k}-x^{s}_{k-1}\right\rangle-\frac{L(\alpha\!-\!1)}{2}\|x^{s}_{k}-x^{s}_{k-1}\|^{2}\right]\\
\leq\,& \mathbb{E}\!\left[\frac{1}{2L(\alpha\!-\!1)}\|\nabla\!f(x^{s}_{k-1})-g_k\|^{2}+\frac{L(\alpha\!-\!1)}{2}\|x^{s}_{k}\!-\!x^{s}_{k-1}\|^{2}-\frac{L(\alpha\!-\!1)}{2}\|x^{s}_{k}\!-\!x^{s}_{k-1}\|^{2}\right]\\
\leq\,& \frac{\delta(b)}{\alpha\!-\!1}\!\left(f(\widetilde{x}^{s-1})-f(x^{s}_{k-1})+\left\langle\nabla f(x^{s}_{k-1}),\;x^{s}_{k-1}-\widetilde{x}^{s-1}\right\rangle\right),
\end{split}
\end{equation}
where the first inequality holds due to the Young's inequality, and the second inequality follows from Lemma 1. Taking the expectation over the random choice of $I_{k}$ and substituting the inequality \eqref{equ72} into the inequality \eqref{equ71}, we have
\begin{equation}\label{equ73}
\begin{split}
\mathbb{E}[f(x^{s}_{k})]&\leq f(x^{s}_{k-1})+\mathbb{E}\!\left[\left\langle g_{k}, \,x^{s}_{k}-x^{s}_{k-1}\right\rangle+\frac{L\alpha}{2}\|x^{s}_{k}-x^{s}_{k-1}\|^2\right]\\
&\quad+\frac{\delta(b)}{\alpha\!-\!1}\!\left[f(\widetilde{x}^{s-1})-f(x^{s}_{k-1})+\left\langle\nabla f(x^{s}_{k-1}),\;x^{s}_{k-1}-\widetilde{x}^{s-1}\right\rangle\right]\\
&= f(x^{s}_{k-1})+\mathbb{E}\!\left[\left\langle g_{k},\;x^{s}_{k}-v^{*}+v^{*}-x^{s}_{k-1}\right\rangle+\frac{L\alpha}{2}\|x^{s}_{k}-x^{s}_{k-1}\|^2\right]\\
&\quad+\frac{\delta(b)}{\alpha\!-\!1}\!\left[f(\widetilde{x}^{s-1})-f(x^{s}_{k-1})+\left\langle\nabla f(x^{s}_{k-1}),\;x^{s}_{k-1}-\widetilde{x}^{s-1}\right\rangle\right]\\
&= f(x^{s}_{k-1})+\mathbb{E}\!\left[\left\langle g_{k}, \;x^{s}_{k}-v^{*}\right\rangle+\frac{L\alpha}{2}\|x^{s}_{k}-x^{s}_{k-1}\|^2\right]+\frac{\delta(b)}{\alpha-1}\left(f(\widetilde{x}^{s-1})-f(x^{s}_{k-1})\right)\\
&\quad+\left\langle\nabla f(x^{s}_{k-\!1}),\;v^{*}\!-\!x^{s}_{k-\!1}\!+\!\frac{\delta(b)}{\alpha\!-\!1}(x^{s}_{k-\!1}\!-\!\widetilde{x}^{s-\!1})\right\rangle\!+\!\mathbb{E}\!\left[\langle\!-\!\frac{1}{b}\!\sum_{i_{k}\in I_{k}}\!\!\nabla\! f_{i_{k}}\!(\widetilde{x}^{s-\!1})\!+\!\nabla\! f(\widetilde{x}^{s-\!1}),\;v^{*}\!-\!x^{s}_{k-\!1}\rangle\right],\\
&= f(x^{s}_{k-1})+\mathbb{E}\!\left[\left\langle g_{k}, \,x^{s}_{k}-v^{*}\right\rangle+\frac{L\alpha}{2}\|x^{s}_{k}-x^{s}_{k-1}\|^2\right]+\frac{\delta(b)}{\alpha\!-\!1}\left(f(\widetilde{x}^{s-1})-f(x^{s}_{k-1})\right)\\
&\quad+\left\langle\nabla f(x^{s}_{k-1}),\;v^{*}-x^{s}_{k-1}+\frac{\delta(b)}{\alpha\!-\!1}(x^{s}_{k-1}-\widetilde{x}^{s-1})\right\rangle,\\
\end{split}
\end{equation}
where $v^{*}\!=\!(1-\theta_{s-1})\widetilde{x}^{s-1}+\theta_{s-1}x^{*}$, the second equality holds due to that $\langle g_{k}, \,v^{*}-x^{s}_{k-1}\rangle=\langle\frac{1}{b}\!\sum_{i_{k}\in I_{k}}\!\!\nabla\! f_{i_{k}}(x^{s}_{k-1}),\,v^{*}-x^{s}_{k-1}\rangle+\langle-\frac{1}{b}\!\sum_{i_{k}\in I_{k}}\!\!\nabla\! f_{i_{k}}(\widetilde{x}^{s-1})+\nabla\! f(\widetilde{x}^{s-1}),\,v^{*}-x^{s}_{k-1}\rangle$ and $\mathbb{E}[\frac{1}{b}\sum_{i_{k}\in I_{k}}\!\!\nabla\! f_{i_{k}}(x^{s}_{k-1})]=\nabla\! f(x^{s}_{k-1})$, and the last equality follows from the fact that $\mathbb{E}\!\left[\langle-\frac{1}{b}\sum_{i_{k}\in I_{k}}\!\!\nabla \! f_{i_{k}}(\widetilde{x}^{s-1})+\nabla\! f(\widetilde{x}^{s-1}),\,v^{*}-x^{s}_{k-1}\rangle\right]=0$.

Furthermore,
\begin{equation}\label{equ75}
\begin{split}
&\,\left\langle\nabla f(x^{s}_{k-1}),\;v^{*}-x^{s}_{k-1}+\frac{\delta(b)}{\alpha\!-\!1}(x^{s}_{k-1}-\widetilde{x}^{s-1})\right\rangle\\
=&\, \left\langle\nabla f(x^{s}_{k-1}),\;(1-\theta_{s-1})\widetilde{x}^{s-1}+\theta_{s-1}x^{*}-x^{s}_{k-1}+\frac{\delta(b)}{\alpha\!-\!1}(x^{s}_{k-1}-\widetilde{x}^{s-1})\right\rangle\\
=&\, \left\langle\nabla f(x^{s}_{k-1}),\;\theta_{s-1}x^{*}+(1-\theta_{s-1}-\frac{\delta(b)}{\alpha\!-\!1})\widetilde{x}^{s-1}+\frac{\delta(b)}{\alpha\!-\!1}x^{s}_{k-1}-x^{s}_{k-1}\right\rangle\\
\leq\, &f\!\left(\theta_{s-1}x^{*}+(1-\theta_{s-1}-\frac{\delta(b)}{\alpha\!-\!1})\widetilde{x}^{s-1}+\frac{\delta(b)}{\alpha\!-\!1}x^{s}_{k-1}\right)-f(x^{s}_{k-1})\\
\leq\, &\theta_{s-1}f(x^{*})+(1-\theta_{s-1}-\frac{\delta(b)}{\alpha\!-\!1})f(\widetilde{x}^{s-1})+\frac{\delta(b)}{\alpha\!-\!1}f(x^{s}_{k-1})-f(x^{s}_{k-1}),
\end{split}
\end{equation}
where the first inequality holds due to the fact that $\langle \nabla f(x),\,y-x\rangle\leq f(y)-f(x)$, and the last inequality follows from the convexity of the function $f$ and the assumption that $1-\theta_{s-1}-\frac{\delta(b)}{\alpha-1}\geq 0$.

Substituting the inequality \eqref{equ75} into the inequality \eqref{equ73}, we have
\begin{equation}\label{equ76}
\begin{split}
\mathbb{E}\!\left[f(x^{s}_{k})\right]\leq&\, f(x^{s}_{k-1})+\mathbb{E}\!\left[\left\langle g_{k}, \,x^{s}_{k}-v^{*}\right\rangle+\frac{L\alpha}{2}\|x^{s}_{k}-x^{s}_{k-1}\|^2\right]+\frac{\delta(b)}{\alpha\!-\!1}\left(f(\widetilde{x}^{s-1})-f(x^{s}_{k-1})\right)\\
&\,+\theta_{s-1}f(x^{*})+(1-\theta_{s-1}-\frac{\delta(b)}{\alpha\!-\!1})f(\widetilde{x}^{s-1})+\frac{\delta(b)}{\alpha\!-\!1}f(x^{s}_{k-1})-f(x^{s}_{k-1})\\
=&\,\theta_{s-1}f(x^{*})+(1-\theta_{s-1})f(\widetilde{x}^{s-1})+\mathbb{E}\!\left[\langle g_{k}, \,x^{s}_{k}-v^{*}\rangle+\frac{L\alpha}{2}\|x^{s}_{k}-x^{s}_{k-1}\|^2\right].
\end{split}
\end{equation}

From the optimality condition of (9) with respect to $z^{s}_{k}$ and $\eta=\frac{1}{L\alpha}$, we have
\begin{equation*}
\begin{split}
\left\langle g_{k}+\beta A^{T}(Az^{s}_{k}+By^{s}_{k}-c)+\beta A^{T}\lambda^{s}_{k-1}+L\alpha\theta_{s-1}G(z^{s}_{k}-z^{s}_{k-1}),\;\,z-z^{s}_{k}\right\rangle\geq 0,\, \textrm{for}\, \textrm{any}\,z\in \mathcal{Z},
\end{split}
\end{equation*}
where $\mathcal{Z}$ is a convex compact set. Since $x^{s}_{k}=\theta_{s-1} z^{s}_{k}+(1-\theta_{s-1})\widetilde{x}^{s-1}$ and $v^{*}=\theta_{s-1}x^{*}+(1-\theta_{s-1})\widetilde{x}^{s-1}$, and the above inequality with $z=x^{*}$, we obtain
\begin{equation}\label{equ7711}
\begin{split}
&\left\langle g_{k},\;x^{s}_{k}-v^{*}\right\rangle=\theta_{s-1}\left\langle g_{k},\;z^{s}_{k}-x^{*}\right\rangle\\
\leq&\,\beta\theta_{s-\!1}\!\left\langle A^{T}(Az^{s}_{k}+By^{s}_{k}-b)+A^{T}\lambda^{s}_{k-1}, \,x^{*}-z^{s}_{k}\right\rangle+L\alpha\theta^{2}_{s-1}\left\langle G(z^{s}_{k}-z^{s}_{k-1}),\,x^{*}-z^{s}_{k}\right\rangle\\
\leq&\,\beta\theta_{s-\!1}\!\left\langle A^{T}(Az^{s}_{k}+By^{s}_{k}-b)+A^{T}\lambda^{s}_{k-1}, \,x^{*}-z^{s}_{k}\right\rangle+\frac{L\alpha\theta^{2}_{s-1}}{2}\left(\|x^{*}-z^{s}_{k-1}\|^{2}_{G}-\|x^{*}-z^{s}_{k}\|^{2}_{G}-\|z^{s}_{k-1}-z^{s}_{k}\|^{2}_{G}\right)\\
\leq&\,\beta\theta_{s-1}\left\langle A^{T}\lambda^{s}_{k}, \,x^{*}-z^{s}_{k}\right\rangle+\frac{L\alpha\theta^{2}_{s-1}}{2}\left(\|x^{*}-z^{s}_{k-1}\|^{2}_{G}-\|x^{*}-z^{s}_{k}\|^{2}_{G}-\|z^{s}_{k-1}-z^{s}_{k}\|^{2}_{G}\right),
\end{split}
\end{equation}
where the second inequity follows from Property 1. Using the optimality condition $\nabla f(x^{*})+\beta A^{T}\lambda^{*}=0$ of problem (2) and let $\varphi^{s}_{k}=\beta\left(\lambda^{s}_{k}-\lambda^{*}\right)$, then
\begin{equation*}
\begin{split}
&\theta_{s-1}\left\langle \beta A^{T}\lambda^{s}_{k},\;x^{*}-z^{s}_{k}\right\rangle\\
= &\,\theta_{s-1}\left\langle\nabla f(x^{*}),\;z^{s}_{k}-x^{*}\right\rangle+\theta_{s-1}\left\langle\beta A^{T}\lambda^{*},\;z^{s}_{k}-x^{*}\right\rangle+\theta_{s-1}\left\langle\beta A^{T}\lambda^{s}_{k}, \;x^{*}-z^{s}_{k}\right\rangle\\
=&\,\theta_{s-1}\left\langle\nabla f(x^{*}),\;z^{s}_{k}-x^{*}\rangle+\theta_{s-1}\langle \varphi^{s}_{k},\;x^{*}-z^{s}_{k}\right\rangle.
\end{split}
\end{equation*}

Taking the expectation of both sides of \eqref{equ7711} over the random choice of $I_{k}$, we have
\begin{equation}\label{equ77}
\begin{split}
&\mathbb{E}\!\left[\left\langle g_{k}, x^{s}_{k}-v^{*}\right\rangle\right]\\
\leq&\,\mathbb{E}\!\left[\theta_{s-\!1}\left\langle\nabla\! f(x^{*}),z^{s}_{k}-x^{*}\right\rangle\!+\!\theta_{s-\!1}\left\langle A^{T}\varphi^{s}_{k},x^{*}\!-z^{s}_{k}\right\rangle\!+\!\frac{L\alpha\theta^{2}_{s-\!1}}{2}\!\left(\|x^{*}\!-z^{s}_{k-\!1}\|^{2}_{G}-\|x^{*}\!-z^{s}_{k}\|^{2}_{G}-\|z^{s}_{k-\!1}\!-z^{s}_{k}\|^{2}_{G}\right)\right].
\end{split}
\end{equation}
Substituting the inequality \eqref{equ77} into the inequality \eqref{equ76}, and $x^{s}_{k}-x^{s}_{k-1}=(1-\theta_{s-1})\widetilde{x}^{s-1}+\theta_{s-1}z^{s}_{k}- (1-\theta_{s-1})\widetilde{x}^{s-1}-\theta_{s-1}z^{s}_{k-1}=\theta_{s-1}(z^{s}_{k}-z^{s}_{k-1})$, we obtain
\begin{equation*}
\begin{split}
&\mathbb{E}\!\left[f(x^{s}_{k})-f(x^{*})-\theta_{s-1}\langle\nabla f(x^{*}),\;z^{s}_{k}-x^{*}\rangle-\theta_{s-1}\langle A^{T}\varphi^{s}_{k},\;x^{*}-z^{s}_{k}\rangle\right]\\
\leq&\, (1\!-\!\theta_{s-1})[f(\widetilde{x}^{s-1})-f(x^{*})]+\frac{L\alpha\theta^{2}_{s-1}}{2}\mathbb{E}\!\left[\|x^{*}-z^{s}_{k-1}\|^{2}_{G}-\|x^{*}-z^{s}_{k}\|^{2}_{G}-\|z^{s}_{k-1}-z^{s}_{k}\|^{2}_{G-I}\right]\\
\leq&\, (1\!-\!\theta_{s-1})[f(\widetilde{x}^{s-1})-f(x^{*})]+\frac{L\alpha\theta^{2}_{s-1}}{2}\mathbb{E}\!\left[\|x^{*}-z^{s}_{k-1}\|^{2}_{G}-\|x^{*}-z^{s}_{k}\|^{2}_{G}\right],
\end{split}
\end{equation*}
where the last inequality holds due to $G\succeq I$ in Algorithms 1 and 2, that is, $\|z^{s}_{k-1}-z^{s}_{k}\|^{2}_{G-I}\geq 0$. Using the update rule $x^{s}_{k}=(1-\theta_{s-1})\widetilde{x}^{s-1}+\theta_{s-1}z^{s}_{k}$ and subtracting $(1\!-\!\theta_{s-1})\langle\nabla f(x^{*}),\;\widetilde{x}^{s-1}-x^{*}\rangle$ from both sides, we have
\begin{equation}\label{equ78}
\begin{split}
&\mathbb{E}\!\left[f(x^{s}_{k})-f(x^{*})-\langle\nabla f(x^{*}),\;x^{s}_{k}-x^{*}\rangle-\theta_{s-1}\langle A^{T}\varphi^{s}_{k},\;x^{*}-z^{s}_{k}\rangle\right]\\
\leq&\, \mathbb{E}\!\left[(1\!-\!\theta_{s-1})\left(f(\widetilde{x}^{s-1})-f(x^{*})-\langle\nabla f(x^{*}),\;\widetilde{x}^{s-1}-x^{*}\rangle\right)+\frac{L\alpha\theta^{2}_{s-1}}{2}\left(\|x^{*}-z^{s}_{k-1}\|^{2}_{G}-\|x^{*}-z^{s}_{k}\|^{2}_{G}\right)\right].
\end{split}
\end{equation}
Since $\widetilde{x}^{s}=\frac{1}{m}\sum^{m}_{k=1}x^{s}_{k}$, and taking the expectation over the random choice of the history of random variables $I_{1},\ldots,I_{m}$ on the inequality \eqref{equ78}, summing it over $k=1,\ldots,m$ at the $s$-th stage and $f(\frac{1}{m}\sum^{m}_{k=1}x^{s}_{k})\leq \frac{1}{m}\sum^{m}_{k=1}f(x^{s}_{k})$, we have
\begin{equation*}
\begin{split}
&\mathbb{E}\!\left[f(\widetilde{x}^{s})-f(x^{*})-\langle\nabla f(x^{*}),\;\widetilde{x}^{s}-x^{*}\rangle-\frac{\theta_{s-1}}{m}\sum^{m}_{k=1}\langle A^{T}\varphi^{s}_{k},\;x^{*}-z^{s}_{k}\rangle\right]\\
\leq&\, \mathbb{E}\!\left[(1\!-\!\theta_{s-1})(f(\widetilde{x}^{s-1})-f(x^{*})-\langle\nabla f(x^{*}),\;\widetilde{x}^{s-1}-x^{*}\rangle)+\frac{L\alpha\theta^{2}_{s-1}}{2m}\left(\|x^{*}-z^{s}_{0}\|^{2}_{G}-\|x^{*}-z^{s}_{m}\|^{2}_{G}\right)\right].
\end{split}
\end{equation*}
This completes the proof.
\end{proof}
\vspace{5mm}

\begin{lemma}\label{lemm3}
Let $\widetilde{y}^{s}=(1-\theta_{s-1})\widetilde{y}^{s-1}+\frac{\theta_{s-1}}{m}\sum^{m}_{k=1}y^{s}_{k}$, then
\begin{displaymath}
\begin{split}
&\mathbb{E}\!\left[h(\widetilde{y}^{s})-h(y^{*})-h'(y^{*})^{T}(\widetilde{y}^{s}-y^{*})-\frac{\theta_{s-1}}{m}\sum^{m}_{k=1}\langle B^{T}\varphi^{s}_{k},\;y^{*}-y^{s}_{k}\rangle\right]\\
\leq &(1-\theta_{s-1})\mathbb{E}\!\left[h(\widetilde{y}^{s-1})-h(y^{*})-h'(y^{*})^{T}(\widetilde{y}^{s-1}-y^{*})\right]\\ &+\frac{\beta\theta_{s-1}}{2m}\mathbb{E}\!\left[\|Az^{s}_{0}+By^{*}-c\|^{2}-\|Az^{s}_{m}+By^{*}-c\|^{2}+\sum^{m}_{k=1}\|\lambda^{s}_{k}-\lambda^{s}_{k-1}\|^{2}\right].
\end{split}
\end{displaymath}
\end{lemma}
\vspace{3mm}

\begin{proof}
Since $\lambda^{s}_{k}=\lambda^{s}_{k-1}+Az^{s}_{k}+By^{s}_{k}-c$, and using Lemma 3 in \citep{zheng:fadmm}, we obtain
\begin{displaymath}
\begin{split}
&\mathbb{E}\!\left[ h(y^{s}_{k})-h(y^{*})-h'(y^{*})^{T}(y^{s}_{k}-y^{*})-\langle B^{T}\varphi^{s}_{k},\;y^{*}-y^{s}_{k}\rangle\right]\\
\leq&\, \frac{\beta}{2}\mathbb{E}\!\left[\|Az^{s}_{k-1}+By^{*}-c\|^{2}-\|Az^{s}_{k}+By^{*}-c\|^{2}+\|\lambda^{s}_{k}-\lambda^{s}_{k-1}\|^{2}\right].
\end{split}
\end{displaymath}
Using the update rule $\widetilde{y}^{s}=(1-\theta_{s-1})\widetilde{y}^{s-1}+\frac{\theta_{s-1}}{m}\sum^{m}_{k=1}y^{s}_{k}$, $h(\widetilde{y}^{s})\leq (1-\theta_{s-1})h(\widetilde{y}^{s-1})+\frac{\theta_{s-1}}{m}\sum^{m}_{k=1}h(y^{s}_{k})$, and taking expectation over whole history and summing the above inequality over $k=1,\ldots,m$, we have
\begin{displaymath}
\begin{split}
&\mathbb{E}\!\left[h(\widetilde{y}^{s})-h(y^{*})-h'(y^{*})^{T}(\widetilde{y}^{s}-y^{*})-\frac{\theta_{s-1}}{m}\sum^{m}_{k=1}\langle B^{T}\varphi^{s}_{k},y^{*}-y^{s}_{k}\rangle\right]\\
\leq&\frac{\theta_{s-1}}{m}\mathbb{E}\!\left[\sum^{m}_{k=1}\left( h(y^{s}_{k})-h(y^{*})-h'(y^{*})^{T}(y^{s}_{k}-y^{*})-\langle B^{T}\varphi^{s}_{k},\;y^{*}-y^{s}_{k}\rangle\right)\right]\\
&+(1-\theta_{s-1})\mathbb{E}\!\left[h(\widetilde{y}^{s-1})-h(y^{*})-h'(y^{*})^{T}(\widetilde{y}^{s-1}-y^{*})\right]\\
\leq &\,\frac{\beta\theta_{s-1}}{2m}\mathbb{E}\!\left[\|Az^{s}_{0}+By^{*}-c\|^{2}-\|Az^{s}_{m}+By^{*}-c\|^{2}+\sum^{m}_{k=1}\|\lambda^{s}_{k}-\lambda^{s}_{k-1}\|^{2}\right]\\
&+(1-\theta_{s-1})\mathbb{E}\!\left[h(\widetilde{y}^{s-1})-h(y^{*})-h'(y^{*})^{T}(\widetilde{y}^{s-1}-y^{*})\right].
\end{split}
\end{displaymath}
This completes the proof.
\end{proof}
\vspace{3mm}

\textbf{Proof of Lemma 2}:
\begin{proof}
Using Lemmas ~\ref{lemm2} and ~\ref{lemm3} and the definition of $P(x,y)$, we have
\begin{equation*}
\begin{split}
&\mathbb{E}\!\left[P(\widetilde{x}^{s},\widetilde{y}^{s})-\frac{\theta_{s-1}}{m}\sum^{m}_{k=1}\left(\langle A^{T}\varphi^{s}_{k},\,x^{*}-z^{s}_{k}\rangle+\langle B^{T}\varphi^{s}_{k},\,y^{*}-y^{s}_{k}\rangle\right)\right]\\
\leq &\, (1-\theta_{s-1})\mathbb{E}\!\left[P(\widetilde{x}^{s-1},\widetilde{y}^{s-1})\right]+\frac{L\alpha\theta^{2}_{s-1}}{2m}\mathbb{E}\!\left[\|x^{*}-z^{s}_{0}\|^{2}_{G}-\|x^{*}-z^{s}_{m}\|^{2}_{G}\right]\\
&+\frac{\beta\theta_{s-1}}{2m}\mathbb{E}\!\left[\|Az^{s}_{0}+By^{*}-c\|^{2}-\|Az^{s}_{m}+By^{*}-c\|^{2}+\sum^{m}_{k=1}\|\lambda^{s}_{k}-\lambda^{s}_{k-1}\|^{2}\right]\\
=&\,\mathbb{E}\!\left[(1\!-\!\theta_{\!s-\!1})P(\widetilde{x}^{s-\!1}\!,\widetilde{y}^{s-\!1})+\frac{\!\theta^{2}_{\!s-\!1}\!\left(\|x^{*}\!-z^{s}_{0}\|^{2}_{G}\!-\!\|x^{*}\!-z^{s}_{m}\|^{2}_{G}\right)}{2m\eta}\right]\\
&+\frac{\beta\theta_{\!s-\!1}}{2m}\mathbb{E}\!\left[\!\|Az^{s}_{0}\!-\!Ax^{*}\|^{2}\!-\!\|Az^{s}_{m}\!\!-\!Ax^{*}\|^{2}\!+\!\!\sum^{m}_{k=1}\!\|\lambda^{s}_{k}\!-\!\lambda^{s}_{k-\!1}\|^{2}\!\right].
\end{split}
\end{equation*}
This completes the proof.
\end{proof}

\section{Proof of Theorem 1:}

Let $(x^{*}\!,y^{*})$ be an optimal solution of the convex problem (2), and $\lambda^{*}$ the corresponding Lagrange multiplier that maximizes the dual. Then $x^{*}$, $y^{*}$ and $\lambda^{*}$ satisfy the following Karush-Kuhn-Tucker (KKT) conditions:
\begin{displaymath}
\beta A^{T}\lambda^{*}+\nabla f(x^{*})=0,\;\;\beta B^{T}\lambda^{*}+h'(y^{*})=0,\;\;Ax^{*}+By^{*}=c.
\end{displaymath}

Before giving the proof of Theorem 1, we first present the following lemmas~\citep{zheng:fadmm}.
\begin{lemma}\label{lemm4}
Let $\varphi_{k}=\beta(\lambda_{k}-\lambda^{*})$, \,and\, $\lambda_{k}=\lambda_{k-1}+Ax_{k}+By_{k}-c$, then
\begin{displaymath}
\begin{split}
\mathbb{E}\!\left[-(Ax_{k}+By_{k}-c)^{T}\varphi_{k}\right]
=\frac{\beta}{2}\mathbb{E}\!\left[\|\lambda_{k-1}-\lambda^{*}\|^{2}-\|\lambda_{k}-\lambda^{*}\|^{2}-\|\lambda_{k}-\lambda_{k-1}\|^{2}\right].
\end{split}
\end{displaymath}
\end{lemma}
\vspace{1mm}

\begin{lemma}\label{lemm6}
Since the matrix $A$ has full row rank, then
\begin{displaymath}
\lambda^{*}=-\frac{1}{\beta}(A^{T})^{\dag}\nabla f(x^{*}).
\end{displaymath}
\end{lemma}

\begin{lemma}\label{lemm7}
Let $\widetilde{\lambda}^{s-1}=\lambda^{s}_{0}=-\frac{1}{\beta}(A^{\dag})\nabla f(\widetilde{x}^{s-1})$, and $\lambda^{*}=-\frac{1}{\beta}(A^{\dag})\nabla f(x^{*})$, then
\begin{equation*}
\|\widetilde{\lambda}^{s-1}-\lambda^{*}\|^{2}\leq \frac{2L_{f}}{\beta^{2}\sigma_{\min}(AA^{T})}\left(f(\widetilde{x}^{s-1})-f(x^{*})-\nabla f(x^{*})^{T}(\widetilde{x}^{s-1}-x^{*})\right).
\end{equation*}
\end{lemma}

\vspace{3mm}
\textbf{Proof of Theorem 1:}
\begin{proof}
Using the update rule $\lambda^{s}_{k}=\lambda^{s}_{k-1}+Az^{s}_{k}+By^{s}_{k}-c$, then
\begin{equation*}
\begin{split}
\sum^{m}_{k=1}\!\left[\langle A^{T}\varphi^{s}_{k},\;x^{*}-z^{s}_{k}\rangle+\langle B^{T}\varphi^{s}_{k},\;y^{*}-y^{s}_{k}\rangle+\langle Az^{s}_{k}+By^{s}_{k}-c,\;\varphi^{s}_{k}\rangle\right]=0,
\end{split}
\end{equation*}
where $\varphi^{s}_{k}=\beta\left(\lambda^{s}_{k}-\lambda^{*}\right)$. Using Lemma ~\ref{lemm4}, we have
\begin{equation*}
\begin{split}
-\sum^{m}_{k=1}\langle Az^{s}_{k}+By^{s}_{k}-c,\;\varphi^{s}_{k}\rangle&=\frac{\beta}{2}\sum^{m}_{k=1}\left(\|\lambda^{s}_{k-1}-\lambda^{*}\|^{2}-\|\lambda^{s}_{k}-\lambda^{*}\|^{2}-\|\lambda^{s}_{k-1}-\lambda^{s}_{k}\|^{2}\right)\\
&=\frac{\beta}{2}\left(\|\lambda^{s}_{0}-\lambda^{*}\|^{2}-\|\lambda^{s}_{m}-\lambda^{*}\|^{2}-\sum^{m}_{k=1}\|\lambda^{s}_{k-1}-\lambda^{s}_{k}\|^{2}\right).
\end{split}
\end{equation*}

Combining Lemma 2 and the above results with $\theta_{s-1}=\theta$ at all stages, $z^{s}_{0}=\widetilde{x}^{s-1}$ and $\lambda^{s}_{0}=\widetilde{\lambda}^{s-1}$, we have
\begin{equation*}
\begin{split}
&\mathbb{E}\!\left[P(\widetilde{x}^{s},\widetilde{y}^{s})\right]\\
\leq&\, (1\!-\!\theta)\mathbb{E}\!\left[P(\widetilde{x}^{s-1},\widetilde{y}^{s-1})\right]+\frac{L\alpha\theta^{2}}{2m}\mathbb{E}\!\left[\|x^{*}-\widetilde{x}^{s-1}\|^{2}_{G}-\|x^{*}-z^{s}_{m}\|^{2}_{G}\right]\\
&+\frac{\beta\theta}{2m}\mathbb{E}\!\left[\|A\widetilde{x}^{s-1}+By^{*}-c\|^{2}-\|Az^{s}_{m}+By^{*}-c\|^{2}+\|\widetilde{\lambda}^{s-1}-\lambda^{*}\|^{2}-\|\lambda^{s}_{m}-\lambda^{*}\|^{2}\right]\\
\leq&\, (1\!-\!\theta)\mathbb{E}\!\left[P(\widetilde{x}^{s-1},\widetilde{y}^{s-1})\right]+\frac{L\alpha\theta^{2}}{2m}\mathbb{E}\!\left[\|x^{*}-\widetilde{x}^{s-1}\|^{2}_{G}\right]+\frac{\beta\theta}{2m}\mathbb{E}\!\left[\|A\widetilde{x}^{s-1}-Ax^{*}\|^{2}+\|\widetilde{\lambda}^{s-1}-\lambda^{*}\|^{2}\right]\\
=&\, (1\!-\!\theta)\mathbb{E}\!\left[P(\widetilde{x}^{s-1},\widetilde{y}^{s-1})\right]+\frac{1}{2m}\mathbb{E}\!\left[\|x^{*}-\widetilde{x}^{s-1}\|^{2}_{L\alpha\theta^{2}G+\beta\theta A^{T}A}\right]+\frac{\beta\theta}{2m}\mathbb{E}\!\left[\|\widetilde{\lambda}^{s-1}-\lambda^{*}\|^{2}\right]\\
\leq&\,(1\!-\!\theta)\mathbb{E}\!\left[P(\widetilde{x}^{s-1},\widetilde{y}^{s-1})\right]+\frac{\|L\alpha\theta^{2}G+\beta\theta A^{T}A\|_{2}}{2m}\mathbb{E}\!\left[\|x^{*}-\widetilde{x}^{s-1}\|^{2}\right]+\frac{\beta\theta}{2m}\mathbb{E}\!\left[\|\widetilde{\lambda}^{s-1}-\lambda^{*}\|^{2}\right]\\
\leq&\,(1\!-\!\theta)\mathbb{E}\!\left[P(\widetilde{x}^{s-\!1},\widetilde{y}^{s-\!1})\right]+\frac{\|L\alpha\theta^{2}G\!+\!\beta\theta A^{T}A\|_{2}}{m\mu}\left(f(\widetilde{x}^{s-\!1})\!-f(x^{*})-\nabla f(x^{*})^{T}(\widetilde{x}^{s-\!1}\!-x^{*})\right)+\frac{\beta\theta}{2m}\mathbb{E}\!\left[\|\widetilde{\lambda}^{s-\!1}\!-\lambda^{*}\|^{2}\right],
\end{split}
\end{equation*}
where the second inequality holds due to the fact that $Ax^{*}+By^{*}-c=0$, and the last inequality follows from the definition of the strongly convex function.

Using $\lambda^{*}=-\frac{1}{\beta}(A^{\dag})\nabla f(x^{*})$, the update rule $\lambda^{s}_{0}=-\frac{1}{\beta}(A^{\dag})\nabla f(\widetilde{x}^{s-1})$, and Lemma ~\ref{lemm7}, we have
\begin{equation*}
\begin{split}
\|\lambda^{s-1}-\lambda^{*}\|^{2}\leq \frac{2L_{f}}{\beta^{2}\sigma_{\min}(AA^{T})}\left[f(\widetilde{x}^{s-1})-f(x^{*})-\nabla f(x^{*})^{T}(\widetilde{x}^{s-1}-x^{*})\right].
\end{split}
\end{equation*}
Combining the above results, $h(\widetilde{y}^{s-1})-h(x^{*})-h'(y^{*})^{T}(\widetilde{y}^{s-1}-y^{*})\geq 0$, $\eta=\frac{1}{L\alpha}$ and the definition of $P(x,y)$, then we have
\begin{equation*}
\begin{split}
&\mathbb{E}\!\left[P(\widetilde{x}^{s},\widetilde{y}^{s})\right]\\
\leq&\,\left(1-\theta+\frac{\|L\alpha\theta^{2}G+\beta\theta A^{T}A\|_{2}}{m\mu}+\frac{L_{f}\theta}{\beta m\sigma_{\min}(AA^{T})}\right)\mathbb{E}\!\left[P(\widetilde{x}^{s-1},\widetilde{y}^{s-1})\right]\\
=&\,\left(1-\theta+\frac{\|\theta^{2}G+\eta\beta\theta A^{T}A\|_{2}}{\eta m\mu}+\frac{L_{f}\theta}{\beta m\sigma_{\min}(AA^{T})}\right)\mathbb{E}\!\left[P(\widetilde{x}^{s-1},\widetilde{y}^{s-1})\right].
\end{split}
\end{equation*}
This completes the proof.
\end{proof}
\vspace{3mm}

\section{Proof of Proposition 1:}
\vspace{1mm}
\begin{proof}
Since
\begin{equation*}
G=\gamma I_{d_{1}}-\frac{\eta\beta^{*} A^{T}\!A}{\theta},\;\;\gamma=\frac{\eta\beta^{*}\|A^{T}\!A\|_{2}}{\theta}+1,\;\, \textup{and} \;\,\beta^{*}\equiv\sqrt{\frac{L_{f}\mu}{\|A^{T}\!A\|_{2}\sigma_{\min}(AA^{T})}},
\end{equation*}
then the rate $\rho$ in Theorem 1 is rewritten as follows:
\begin{equation}\label{equ1180}
\begin{split}
\rho(\theta)\,&=\frac{\theta\|\theta G\!+\!\eta\beta^{*}A^{T}\!A\|_{2}}{\eta m\mu}\!+1-\theta+\frac{L_{f}\theta}{\beta^{*} m\sigma_{\min}(AA^{T})}\\
&=\theta(\eta\beta^{*}\|A^{T}\!A\|_{2}+\theta)/(\eta m\mu)+1-\theta+\frac{L_{f}\theta\sqrt{\|A^{T}\!A\|_{2}/\sigma_{\min}(AA^{T})}}{m\sqrt{L_{f}\mu}}\\
&=(\theta^{2}/\eta+\beta^{*}\theta\|A^{T}\!A\|_{2})/(m\mu)+1-\theta+\frac{\theta\sqrt{L_{f}/\mu}\sqrt{\omega}}{m}\\
&=\frac{\kappa\alpha\theta^{2}}{m}+\frac{2\theta\sqrt{\kappa_{\!f}\omega}}{m}+1-\theta,
\end{split}
\end{equation}
where $\alpha=\frac{1}{L\eta}$. Therefore, the rate $\rho$ can be expressed as a simple convex function with respect to $\theta$ by fixing other variable. To minimize the quadratic function $\rho$ with respect to $\theta$, then we have
\begin{equation*}
\theta^{*}=\frac{m-2\sqrt{\kappa_{\!f}\omega}}{2\kappa\alpha}.
\end{equation*}

Recall from the body of this paper that $\alpha=\frac{m-2\sqrt{\kappa_{\!f} \omega}}{2\kappa}+\delta(b)+1$. Then it is not difficult to verify that
\begin{equation*}
\begin{split}
\theta^{*}=&\,\frac{m-2\sqrt{\kappa_{\!f}\omega}}{2\kappa\alpha}=\frac{m-2\sqrt{\kappa_{\!f}\omega}}{[(m-2\sqrt{\kappa_{\!f} \omega})+2\kappa(\delta(b)+1)]}\\
\leq&\,\frac{m-2\sqrt{\kappa_{\!f}\omega}}{[(m-2\sqrt{\kappa_{\!f} \omega})+2\kappa\delta(b)]}=1-\frac{\delta(b)}{[(m-2\sqrt{\kappa_{\!f} \omega})/(2\kappa)+\delta(b)]}\\
=&\,1-\frac{\delta(b)}{\alpha-1}.
\end{split}
\end{equation*}

By the above result, we have
\begin{equation}
\theta^{*}=\frac{m-2\sqrt{\kappa_{\!f}\omega}}{2\kappa\alpha}=\frac{m-2\sqrt{\kappa_{\!f}\omega}}{m-2\sqrt{\kappa_{\!f}\omega}+2\kappa(\delta(b)+1)}.
\end{equation}
It is not difficult to verify that $0<\rho(\theta^{*})<1$ is satisfied. Thus, $\theta^{*}$ is the optimal solution of (\ref{equ1180}) with $\alpha=\frac{m-2\sqrt{\kappa_{\!f} \omega}}{2\kappa}+\delta(b)+1$.

This proof is completed.
\end{proof}
\vspace{3mm}

\section{Proof of Theorem 2:}
Before giving the proof of Theorem 2, we first present the following lemma~\citep{zheng:fadmm}.
\begin{lemma}\label{lemm5}
Let $\varphi_{k}=\beta(\lambda_{k}-\lambda^{*})$ and any $\varphi=\beta\lambda$, and $\lambda_{k}=\lambda_{k-1}+Ax_{k}+By_{k}-c$, then
\begin{displaymath}
\mathbb{E}\!\left[-(Ax_{k}+By_{k}-c)^{T}(\varphi_{k}-\varphi)\right]=\frac{\beta}{2}\mathbb{E}\!\left[\|\lambda_{k-1}-\lambda^{*}-\lambda\|^{2}-\|\lambda_{k}-\lambda^{*}-\lambda\|^{2}-\|\lambda_{k}-\lambda_{k-1}\|^{2}\right].
\end{displaymath}
\end{lemma}
\vspace{3mm}

\vspace{1mm}
\textbf{Proof of Theorem 2:}
\begin{proof}
For any $\varphi=\beta\lambda$, we have
\begin{equation*}
\begin{split}
\sum^{m}_{k=1}\left(\langle A^{T}\varphi^{s}_{k},\,x^{*}-z^{s}_{k}\rangle+\langle B^{T}\varphi^{s}_{k},\,y^{*}-y^{s}_{k}\rangle+\langle Az^{s}_{k}+By^{s}_{k}-c,\,\varphi^{s}_{k}-\varphi\rangle\right)=-\sum^{m}_{k=1}\langle Az^{s}_{k}+By^{s}_{k}-c,\,\varphi\rangle,
\end{split}
\end{equation*}
where $\varphi^{s}_{k}=\beta\left(\lambda^{s}_{k}-\lambda^{*}\right)$. Using Lemma~\ref{lemm5}, we have
\begin{equation*}
\begin{split}
-\sum^{m}_{k=1}\langle Az^{s}_{k}+By^{s}_{k}-c,\,\varphi^{s}_{k}-\varphi\rangle&=\frac{\beta}{2}\sum^{m}_{k=1}\left(\|\lambda^{s}_{k-1}-\lambda^{*}-\lambda\|^{2}-\|\lambda^{s}_{k}-\lambda^{*}-\lambda\|^{2}-\|\lambda^{s}_{k-1}-\lambda^{s}_{k}\|^{2}\right)\\
&=\frac{\beta}{2}\left(\|\lambda^{s}_{0}-\lambda^{*}-\lambda\|^{2}-\|\lambda^{s}_{m}-\lambda^{*}-\lambda\|^{2}-\sum^{m}_{k=1}\|\lambda^{s}_{k-1}-\lambda^{s}_{k}\|^{2}\right).
\end{split}
\end{equation*}

Combining the above results and Lemma 2, we have
\begin{equation}
\begin{split}
&\mathbb{E}\!\left[P(\widetilde{x}^{s}\!,\widetilde{y}^{s})-\frac{\theta_{s-1}}{m}\sum^{m}_{k=1}\langle Az^{s}_{k}+By^{s}_{k}-c,\;\varphi\rangle\right]\\
\leq&\, (1\!-\!\theta_{s-1})\mathbb{E}\!\left[P(\widetilde{x}^{s-1}\!,\widetilde{y}^{s-1})\right]+\frac{L\alpha\theta^{2}_{s-1}}{2m}\mathbb{E}\!\left[\|x^{*}-z^{s}_{0}\|^{2}_{G}-\|x^{*}-z^{s}_{m}\|^{2}_{G}\right]\\
&+\frac{\beta\theta_{s-1}}{2m}\mathbb{E}\!\left[\|Az^{s}_{0}+By^{*}-c\|^{2}-\|Az^{s}_{m}+By^{*}-c\|^{2}+\|\lambda^{s}_{0}-\lambda^{*}-\lambda\|^{2}-\|\lambda^{s}_{m}-\lambda^{*}-\lambda\|^{2}\right].
\end{split}
\end{equation}

Using the update rule of $\widetilde{x}^{s}=\frac{1}{m}\sum^{m}_{k=1}x^{s}_{k}$, we have
\begin{equation*}
\begin{split}
\widetilde{x}^{s}=\frac{1}{m}\sum^{m}_{k=1}x^{s}_{k}=\frac{1}{m}\sum^{m}_{k=1}\left(\theta_{s-1}z^{s}_{k}+(1-\theta_{s-1})\widetilde{x}^{s-1}\right)=(1-\theta_{s-1})\widetilde{x}^{s-1}+\frac{\theta_{s-1}}{m}\sum^{m}_{k=1}z^{s}_{k}.
\end{split}
\end{equation*}
Recall that
\begin{equation*}
\widetilde{y}^{s}=(1-\theta_{s-1})\widetilde{y}^{s-1}+\frac{\theta_{s-1}}{m}\sum^{m}_{k=1}y^{s}_{k}.
\end{equation*}
Subtracting $(1\!-\!\theta_{s-1})\langle A\widetilde{x}^{s-1}+B\widetilde{y}^{s-1}-c,\,\varphi\rangle$ and dividing both sides by $\theta^{2}_{s-1}$, we have
\begin{equation}
\begin{split}
&\frac{1}{\theta^{2}_{s-1}}\mathbb{E}\!\left[P(\widetilde{x}^{s}\!,\widetilde{y}^{s})-\langle A\widetilde{x}^{s}+B\widetilde{y}^{s}-c,\,\varphi\rangle\right]\\
\leq&\, \frac{(1\!-\!\theta_{s-1})}{\theta^{2}_{s-1}}\mathbb{E}\!\left[P(\widetilde{x}^{s-1}\!,\widetilde{y}^{s-1})-\langle A\widetilde{x}^{s-1}+B\widetilde{y}^{s-1}-c,\,\varphi\rangle\right]+\frac{L\alpha}{2m}\mathbb{E}\!\left[\|x^{*}-z^{s}_{0}\|^{2}_{G}-\|x^{*}-z^{s}_{m}\|^{2}_{G}\right]\\
&+\frac{\beta}{2m\theta_{s-1}}\mathbb{E}\!\left[\|Az^{s}_{0}+By^{*}-c\|^{2}-\|Az^{s}_{m}+By^{*}-c\|^{2}+\|\lambda^{s}_{0}-\lambda^{*}-\lambda\|^{2}-\|\lambda^{s}_{m}-\lambda^{*}-\lambda\|^{2}\right].
\end{split}
\end{equation}

By the update rule in (12), we have $(1-\theta_{s})/{\theta^{2}_{s}}={1}/{\theta^{2}_{s-1}}$. Since $z^{s}_{0}=z^{s-1}_{m}$, $\widetilde{z}^{0}=\widetilde{x}^{0}$, and summing over all stages $(s=1,\cdots,T)$, we have
\begin{equation}\label{equ110}
\begin{split}
&\frac{1}{\theta^{2}_{T-1}}\mathbb{E}\left[P(\widetilde{x}^{T}\!,\widetilde{y}^{T})-\langle A\widetilde{x}^{T}+B\widetilde{y}^{T}-c,\,\varphi\rangle\right]\\
\leq&\, \frac{(1\!-\!\theta_{0})}{\theta^{2}_{0}}\mathbb{E}\left[P(\widetilde{x}^{0}\!,\widetilde{y}^{0})-\langle A\widetilde{x}^{0}+B\widetilde{y}^{0}-c,\,\varphi\rangle\right]+\frac{L\alpha}{2m}\mathbb{E}\!\left[\|x^{*}-\widetilde{x}^{0}\|^{2}_{G}-\|x^{*}-z^{0}_{m}\|^{2}_{G}\right]\\
&+\sum^{T}_{s=1}\frac{\beta}{2m\theta_{s-1}}\mathbb{E}\!\left[\|Az^{s}_{0}-A\widetilde{x}^{*}\|^{2}-\|Az^{s}_{m}-A\widetilde{x}^{*}\|^{2}+\|\lambda^{s}_{0}-\lambda^{*}-\lambda\|^{2}-\|\lambda^{s}_{m}-\lambda^{*}-\lambda\|^{2}\right].
\end{split}
\end{equation}

Since $(1-\theta_{s})/{\theta^{2}_{s}}={1}/{\theta^{2}_{s-1}}$, then
\begin{equation*}
\begin{split}
0\leq \frac{1}{\theta_{s}}-\frac{1}{\theta_{s-1}}\leq \frac{1}{\theta_{s}}-\frac{\sqrt{1-\theta_{s}}}{\theta_{s}}=\frac{1}{1+\sqrt{1-\theta_{s}}}<1.
\end{split}
\end{equation*}

Using $z^{s}_{0}=z^{s-1}_{m}$ and $\lambda^{s}_{0}=\lambda^{s-1}_{m}$, we have

\begin{equation}\label{equ111}
\begin{split}
&\sum^{T}_{s=1}\frac{\beta}{2m\theta_{s-1}}\mathbb{E}\!\left[\|Az^{s}_{0}-A\widetilde{x}^{*}\|^{2}-\|Az^{s}_{m}-A\widetilde{x}^{*}\|^{2}+\|\lambda^{s}_{0}-\lambda^{*}-\lambda\|^{2}-\|\lambda^{s}_{m}-\lambda^{*}-\lambda\|^{2}\right]\\
\leq &\mathbb{E}\!\left[\frac{\beta(1-\theta_{0})}{2m\theta_{0}}(\|Az^{1}_{0}-A\widetilde{x}^{*}\|^{2}+\|\lambda^{1}_{0}-\lambda^{*}-\lambda\|^{2})+\frac{\beta}{2m}\sum^{T}_{s=1}\left(\|Az^{s}_{0}-A\widetilde{x}^{*}\|^{2}+\|\lambda^{s}_{0}-\lambda^{*}-\lambda\|^{2}\right)\right]\\
\leq &\frac{\beta (T-1+1/\theta_{0})}{2m}\mathbb{E}\!\left[\|A^{T}\!A\|_{2}D^{2}_{\mathcal{Z}}+4D^{2}_{\lambda}\right],
\end{split}
\end{equation}
where the first inequality follows from that ${1}/{\theta_{s}}-{1}/{\theta_{s-1}}<1$, $z^{1}_{0}=\widetilde{x}^{0}$ and $\lambda^{1}_{0}=\widetilde{\lambda}^{0}$, and the last inequality holds due to that $\|Az^{s}_{0}-A\widetilde{x}^{*}\|^{2}\leq \|A^{T}\!A\|_{2}D^{2}_{\mathcal{Z}}$ and $\|\lambda^{s}_{0}-\lambda^{*}-\lambda\|\leq \|\lambda^{s}_{0}-\lambda^{*}\|+\|\lambda\|\leq 2D_{\lambda}$.
\vspace{1mm}

By \eqref{equ110} and \eqref{equ111} with $\theta_{s}\leq{2}/(s+2)$ for all $s$, we have
\begin{equation*}
\begin{split}
&\mathbb{E}\!\left[P(\widetilde{x}^{T}\!,\widetilde{y}^{T})-\langle A\widetilde{x}^{T}+B\widetilde{y}^{T}-c,\,\varphi\rangle\right]\\
\leq&\, \frac{4\tau}{(T\!+\!1)^{2}}\left(P(\widetilde{x}^{0}\!,\widetilde{y}^{0})-\langle A\widetilde{x}^{0}+B\widetilde{y}^{0}-c,\,\varphi\rangle\right)+\frac{2L\alpha}{m(T\!+\!1)^{2}}\|x^{*}-\widetilde{x}^{0}\|^{2}_{G}+\frac{2\beta(T\!-\!1\!+\!1/\theta_{0})}{m(T\!+\!1)^{2}}\left(\|A^{T}\!A\|_{2}D^{2}_{\mathcal{Z}}+4D^{2}_{\lambda}\right),\\
\leq&\, \frac{4\tau}{(T\!+\!1)^{2}}\left(P(\widetilde{x}^{0}\!,\widetilde{y}^{0})-\langle A\widetilde{x}^{0}+B\widetilde{y}^{0}-c,\,\varphi\rangle\right)+\frac{2L\alpha}{m(T\!+\!1)^{2}}\|x^{*}-\widetilde{x}^{0}\|^{2}_{G}+\frac{2\beta}{m(T\!+\!1)}\left(\|A^{T}\!A\|_{2}D^{2}_{\mathcal{Z}}+4D^{2}_{\lambda}\right),
\end{split}
\end{equation*}
where $\theta_{0}=1-\frac{\delta(b)}{\alpha-1}$ and $\tau=(1-\theta_{0})/\theta^{2}_{0}=\frac{(\alpha-1)\delta(b)}{(\alpha-1-\delta(b))^{2}}$.  Setting $\varphi=\gamma\frac{A\widetilde{x}^{T}+B\widetilde{y}^{T}-c}{\|A\widetilde{x}^{T}+B\widetilde{y}^{T}-c\|}$ with $\gamma\leq \beta D_{\lambda}$ such that $\|\lambda\|=\|\varphi/\beta\|\leq D_{\lambda}$, and $-\langle A\widetilde{x}^{0}+B\widetilde{y}^{0}-c,\varphi\rangle\leq \|\varphi\|\|A\widetilde{x}^{0}+B\widetilde{y}^{0}-c\|\leq \gamma\|A\widetilde{x}^{0}+B\widetilde{y}^{0}-c\|$, we have

\begin{equation*}
\begin{split}
&\mathbb{E}\!\left[P(\widetilde{x}^{T}\!,\widetilde{y}^{T})+\gamma\|A\widetilde{x}^{T}+B\widetilde{y}^{T}-c\|\right]\\
\leq&\, \frac{4(\alpha-1)\delta(b)}{(\alpha-1-\delta(b))^{2}(T+1)^{2}}\left(P(\widetilde{x}^{0},\widetilde{y}^{0})+ \gamma\|A\widetilde{x}^{0}+B\widetilde{y}^{0}-c\|\right)+\frac{2L\alpha}{m(T+1)^{2}}\|x^{*}-\widetilde{x}^{0}\|^{2}_{G}\\
&+\frac{2\beta}{m(T+1)}\left(\|A^{T}\!A\|_{2}D^{2}_{\mathcal{Z}}+4D^{2}_{\lambda}\right).
\end{split}
\end{equation*}
This completes the proof.
\end{proof}

\end{document}